\documentclass{article} 
\usepackage{nips15submit_e,times}
\usepackage{hyperref}
\usepackage{url}
\pdfoutput=1

\usepackage{graphicx} 
\usepackage{algorithm}
\usepackage[noend]{algorithmic}

\usepackage{amssymb}
\usepackage{amsmath}
\usepackage{color}
\usepackage{enumitem}

\usepackage{mathtools}
\DeclarePairedDelimiter\ceil{\lceil}{\rceil}
\DeclarePairedDelimiter\floor{\lfloor}{\rfloor}

\usepackage{breqn}
\usepackage{bbm}

\usepackage{caption}
\usepackage{subcaption}
\usepackage{amsmath}

\usepackage{float}
\usepackage{makeidx}

\usepackage{array}
\usepackage{eqparbox}

\title{Bayesian Optimization with Exponential Convergence}

\author{
Kenji Kawaguchi \\
MIT  \\
Cambridge, MA, 02139\\
\texttt{kawaguch@mit.edu} \\
\And
Leslie Pack Kaelbling \\
MIT \\
Cambridge, MA, 02139 \\
\texttt{lpk@csail.mit.edu} \\
\And
Tom\'{a}s Lozano-P\'{e}rez \\
MIT \\
Cambridge, MA, 02139 \\
\texttt{tlp@csail.mit.edu} \\
}

\usepackage{amsthm}
\theoremstyle{definition}
\newtheorem{theorem}{Theorem}
\newtheorem{lemma}{Lemma}

\newtheorem{corollary}{Corollary}
\newtheorem{defn}{Definition}
\newtheorem{remark}{Remark}
\newtheorem{assum}{Assumption}

\nipsfinalcopy 

\begin{document}

\maketitle

\begin{abstract}
This paper presents a Bayesian optimization method with exponential convergence  \textit{without} the need of  auxiliary optimization and \textit{without} the \(\delta\)-cover sampling. Most Bayesian optimization methods require  auxiliary optimization: an additional non-convex global optimization problem, which can be time-consuming and hard to implement in practice. \hspace{-3.5pt}Also, the existing Bayesian optimization method with exponential convergence \cite{de2012exponential} requires access to the \(\delta\)-cover sampling, which was considered to be impractical \cite{de2012exponential,wang2014bayesian}. Our approach eliminates both requirements and achieves an exponential convergence rate.

\end{abstract}

\section{Introduction}
\label{submission}
We consider a general global optimization problem: maximize $f( x) \: $ subject to $ x \in \Omega \subset \mathbb R^D$ where $f\!: \Omega   \to\mathbb{R}$  is a non-convex  black-box deterministic function. Such a problem arises in many real-world applications, such as parameter tuning in machine learning  \cite{snoek2012practical}, engineering design problems \cite{carter2001algorithms}, and model parameter fitting in biology \cite{zwolak2005globally}. For this problem, one performance measure  of an algorithm  is the \textit{simple regret}, \(r_n\), which is given by \(r_n = \sup_{ x \in \Omega} f( x) - f( x^+) \)
where \( x^+\) is the best input vector found by the algorithm. For brevity, we use the term ``regret" to mean simple regret.

The general global optimization problem is known to be intractable if we make no further assumptions \cite{dixon1977global}.  The simplest  additional assumption to restore tractability is to assume the existence of a bound on the  slope of \(f\). A well-known variant of this assumption  is Lipschitz continuity with a known Lipschitz constant, and many algorithms have been proposed in this setting \cite{shubert1972sequential, mayne1984outer,mladineo1986algorithm}. These algorithms successfully guaranteed certain bounds on the regret. However appealing from a theoretical point of view, a practical concern was soon raised regarding
the assumption that a tight Lipschitz constant is known. Some researchers
relaxed this somewhat strong assumption by proposing procedures to estimate a Lipschitz
constant during the optimization process \cite{strongin1973convergence,kvasov2003local,bubeck2011lipschitz}.

Bayesian optimization is an efficient way to relax this assumption of complete knowledge of the Lipschitz constant,  and has become a well-recognized method for solving global optimization
problems with non-convex black-box functions. In the machine learning community, Bayesian
optimization---especially by means of a Gaussian process (GP)---is an active research area \cite{gardner2014bayesian,wang2013bayesian,srinivas2010gaussian}.  With the requirement of the access to the \(\delta\)-cover sampling procedure (it samples  the function uniformly such that the density of samples doubles in  the feasible regions at each iteration), de  Freitas et al. \cite{de2012exponential}  recently proposed a theoretical procedure that maintains an exponential convergence rate (exponential regret). However, as pointed out by Wang et al. \cite{wang2014bayesian}, one  remaining problem is to derive a GP-based optimization method with an exponential convergence rate  \textit{without} the \(\delta\)-cover sampling procedure, which is computationally too demanding in many cases.

In this paper, we propose a novel GP-based global optimization algorithm, which maintains an exponential convergence rate and converges rapidly \textit{without} the \(\delta\)-cover sampling procedure.

\section{Gaussian Process Optimization}
In Gaussian process optimization, we estimate the distribution over function \(f\) and use this information to decide which point of \(f\) should be evaluated next. In a parametric approach, we consider a parameterized function \(f({x;} \: \theta)\),  with  \(\theta\) being distributed according to some prior. In contrast, the nonparametric GP approach directly puts the GP prior over \(f\) as $ f({\cdot}) \sim  GP (m({\cdot)}, {\mkern 1mu} {\kappa({\cdot,\cdot})}) $
where \(m({\cdot})\) is the mean function and \(\kappa({\cdot,\cdot})\) is the covariance function or the kernel. That is, \(m({x})=\mathbb{E}[f({x})]\) and \(\kappa({x,x'})=\mathbb{E}[(f({x})-m({x}))(f({x'})-m({x'}))^{T}]\). For a finite set of points, the GP model is simply a joint Gaussian: \(\mathbf{f(x}_{1:N}) \sim {\cal N}(\mathbf{m(x}_{1:N}),\mathbf{K})\), where \(\mathbf{K}_{i,j}=\kappa( x_i, x_j)\) and \(N\) is the number of data points. To predict the value of \(f\) at  a new data point, we first consider the joint distribution over \(f\) of  the old data points and the new data point:

\begin{align*}
\left( {\begin{array}{c}
\mathbf{f(x}_{1:N}) \\
f(x_{N+1}) \\
\end{array}} \right) \sim {\cal N}
\left( \begin{array}{c}
\mathbf{m(x}_{1:N}) \\
m(x_{N+1}) \\
\end{array}
,
\left[
\begin{array}{cc}
\mathbf{K} & \mathbf{k} \\
\mathbf{k}^{T} & \kappa(x_{N+1},x_{N+1}) \\
\end{array}
\right]
\right)
\end{align*}
where \(\mathbf{k}=\kappa(\mathbf x_{1:N},\mathbf x_{N+1})\in\mathbb{R}^{N\times1}\). Then,  after factorizing the joint distribution using the Schur complement for the joint Gaussian, we obtain the conditional distribution, conditioned on observed entities \({\cal D}_N := \{\mathbf x_{1:N},\mathbf{f(x}_{1:N}) \}\) and \(x_{N+1}\),  as:
\[
f(\mathbf x_{N+1})| {\cal D}_{N}, x_{N+1} \sim {\cal N}(\mu ( x_{N+1}|{\cal D}_{N}),\sigma^2( x_{N+1}|{\mathcal D_{N}}) )
\]
where \(\mu( x_{N+1}|{\cal D}_{N}) = m( x_{N+1})+\mathbf{k}^{T}\mathbf{K}^{-1}(\mathbf{f(x}_{1:N})-\mathbf{m(x}_{1:N})) \) and \(\sigma^2( x_{N+1}|{\cal D}_{N})=\kappa(\mathbf{x}_{N+1},\mathbf{x}_{N+1})-\mathbf{k}^{T}\mathbf{K}^{-1}\mathbf{k}\). One advantage of GP is that this closed-form solution simplifies both its analysis and  implementation.

To use a GP, we must specify the mean function and the covariance function. The mean function is usually set to  be zero. With this zero mean function, the conditional mean \(\mu( x_{N+1}|{\cal D}_{N})\) can  still be flexibly specified by the covariance function, as shown in the above equation for  \(\mu\).  For the covariance function, there are several common choices, including the Matern kernel and the Gaussian kernel. For example, the Gaussian kernel is defined as $ \kappa({x,x'})= \text{exp} \left( -\frac1{2}{(x-x')}^T{\Sigma^{-1}(x-x')} \right) $ where \( \Sigma^{-1}\) is the kernel parameter matrix. The kernel parameters or hyperparameters can be estimated by empirical Bayesian methods \cite{murphy2012machine}; see \cite{gpml06} for more information about GP.

The flexibility and simplicity of the GP prior make it a common choice for continuous objective functions in the Bayesian optimization literature. Bayesian optimization with GP selects the next query point that optimizes the acquisition function generated by GP. Commonly used acquisition functions include the upper confidence bound (UCB) and expected improvement (EI). For brevity, we consider Bayesian optimization with UCB, which works as follows. At each iteration, the UCB function \({\cal U}\) is maintained as $ {\cal U}( x|D_{N})=\mu( x|{\cal D}_{N})+ \varsigma \sigma( x|{\cal D}_{N}) $
where $\varsigma \in \mathbb R$ is a parameter of the algorithm. To find the next query \( x_{n+1}\) for the objective function \(f\), GP-UCB solves an additional non-convex optimization problem with \({\cal U}\) as \( x_{N+1}=\arg\max_{ x} {\cal U}( x|D_N)\). This is often carried out by other global optimization methods such as DIRECT and CMA-ES. The justification for introducing a new optimization problem lies in the assumption that the cost of evaluating the objective function \(f\) dominates  that of solving additional optimization problem.

For deterministic function, de Freitas et al. \cite{de2012exponential} recently presented a theoretical procedure that maintains exponential convergence rate. However, their  own paper and the follow-up research \cite{de2012exponential,wang2014bayesian} point out that this result relies  on an impractical sampling procedure,  the \(\delta\)-cover sampling. To overcome this issue, Wang et al. \cite{wang2014bayesian} combined GP-UCB with a hierarchical partitioning optimization method, the SOO algorithm \cite{munos2011optimistic}, providing a regret bound with polynomial dependence on the number of function evaluations. They concluded that creating a  GP-based algorithm with an \textit{exponential} convergence rate \textit{without} the impractical sampling procedure remained an open problem.

\section{Infinite-Metric GP Optimization}
\subsection{Overview}
The GP-UCB algorithm can be seen as a member of the class of  bound-based search methods, which includes Lipschitz optimization, A* search, and PAC-MDP algorithms with optimism in the face of uncertainty. Bound-based search methods have a common property: the tightness of the bound
determines its effectiveness. The tighter the bound is, the better the performance becomes. However, it is often difficult to obtain a tight bound while maintaining correctness. For example, in A* search, admissible heuristics maintain the correctness of the bound, but  the estimated bound with admissibility is often too loose in practice, resulting in a  long period of global search.

The GP-UCB algorithm
has the same problem. The  bound in GP-UCB is represented by UCB, which has the following property: \(f( x) \leq {\cal U} ( x|{\cal D})\) with some probability. We formalize this property in the analysis of our algorithm. The   problem is essentially due to the difficulty of obtaining a tight bound \({\cal U} ( x|{\cal D})\) such that  \(f( x) \leq {\cal U} ( x|{\cal D})\) and \(f( x) \approx {\cal U} ( x|{\cal D})\) (with some probability). Our solution strategy is to first admit that the bound encoded in GP prior may not be  tight enough to be useful by itself. Instead of  relying on a single bound given by the GP, we leverage the existence of an \textit{unknown}  bound encoded in the continuity  at a global optimizer.
\begin{assum}
(Unknown Bound) There exists a global optimizer \( x^ *  \) and  an \textit{unknown} semi-metric \(\ell\) such that for all \( x \in \Omega\),  $ f({ x^ * })  \le  f( x) + \ell \left( { x,{ x^ * }} \right)$ and $\ell \left( { x,{ x^ * }} \right) < \infty$.
\end{assum}

In other words, we do not expect the \textit{known} upper  bound due to GP to be tight, but instead expect  that there exists some \textit{unknown}  bound that might be tighter. Notice that in the case where the bound by GP is as tight as the unknown bound by semi-metric \(\ell\) in Assumption 1, our method still maintains an exponential convergence rate and an advantage over GP-UCB (no need for auxiliary optimization). Our method is expected to become  relatively much better when the \textit{known}   bound due to GP is less tight compared to the unknown bound by \(\ell\).

As the semi-metric \(\ell\) is unknown, there
are infinitely many possible candidates that we can think of for \(\ell\). Accordingly, we simultaneously conduct  global and local searches based on all the candidates of the bounds. The bound estimated by GP is used to reduce the number of  candidates. Since the bound estimated by GP is known, we can ignore the candidates of  the bounds that are looser than  the bound estimated by GP. The source code of the proposed algorithm is publicly available at \url{http://lis.csail.mit.edu/code/imgpo.html}.

\subsection{Description of Algorithm}
Figure 1 illustrates how the algorithm works with a simple 1-dimensional objective function. We employ hierarchical partitioning  to maintain hyperintervals, as illustrated by the line segments in the figure. We consider a hyperrectangle as our hyperinterval, with its center being the evaluation point of \(f\) (blue points in each line segment in Figure 1). For each iteration $t$, the algorithm performs the following procedure \textit{for each interval size}:
\begin{enumerate}[label=(\roman*)]
\item Select the interval with the maximum center value  among the intervals of the same size.
\item Keep the  interval selected by (i)  if it has a center value greater than that of any \textit{larger}  interval.
\item Keep the interval  accepted by (ii) if it contains a UCB greater than the center value of any \textit{smaller} interval.
\item If an interval is accepted by  (iii), divide it along  with the longest coordinate into three new intervals.
\item  For each new interval,  if the UCB of the evaluation point is less than the best  function value found so far, skip the evaluation and use the UCB value as the center value until the interval is accepted in step (ii) on some future iteration;  otherwise, evaluate the center value.
\item Repeat steps (i)--(v) until every size of intervals are considered
\end{enumerate}
Then, at the end of each iteration, the algorithm   updates the GP hyperparameters. Here, the purpose of steps (i)--(iii) is to select an interval that might contain the global optimizer. Steps (i) and (ii) select the possible intervals based on the unknown bound by \(\ell\), while Step (iii) does so based on the bound by GP.

\begin{figure*}[t!]
   \includegraphics[width=\linewidth]{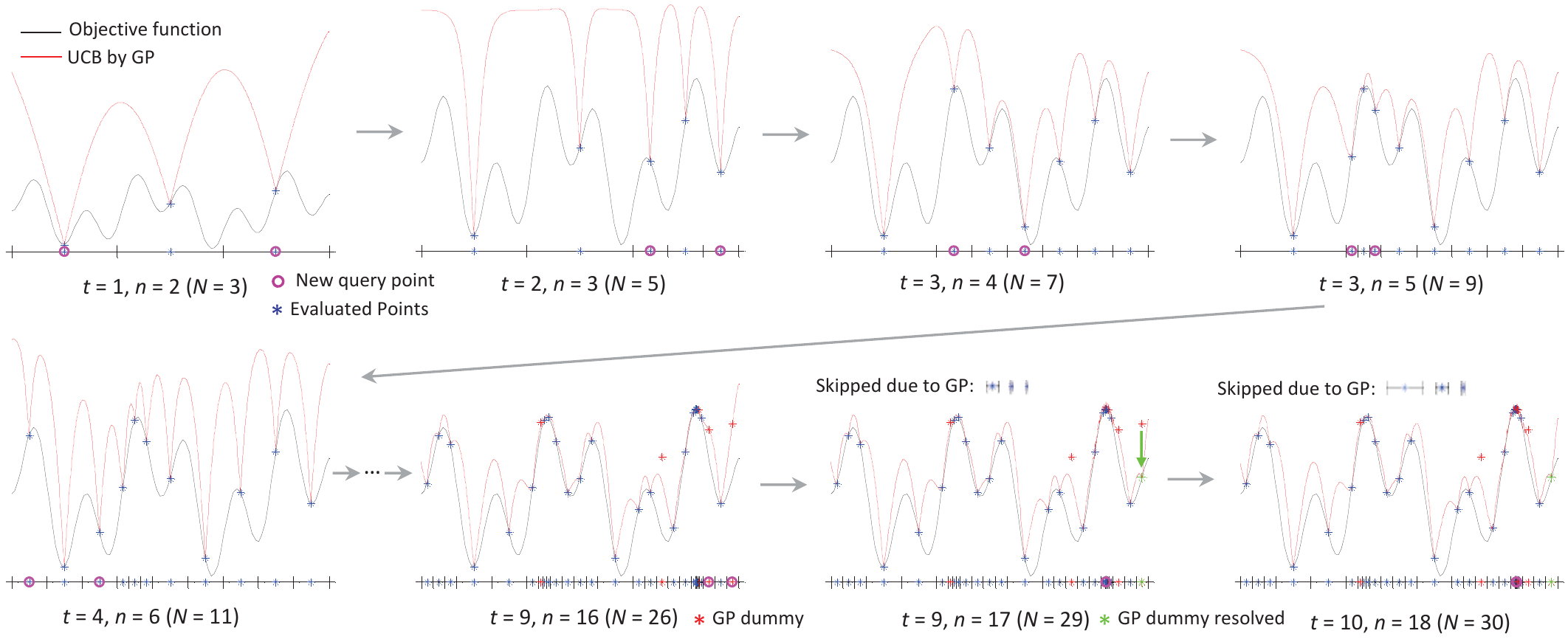}
   \caption{An illustration of IMGPO: \(t\) is the number of iteration, \(n\) is the number of divisions (or splits), \(N\) is the number of function evaluations.}
\end{figure*}

We now explain the procedure using the example in Figure 1. Let \(n\) be the number of divisions of intervals and let \(N\) be the number of function evaluations. \(t\) is the number of iterations. Initially, there is only one interval (the center of the input region \(\Omega \subset \mathbb R\)) and thus this interval is divided, resulting in the first diagram of Figure 1. At the beginning of iteration $t=2$ ,  step  (i) selects the third  interval  from the left side in the first diagram ($t = 1, n=2)$, as its center value is the maximum. Because there are no intervals of different size at this point, steps (ii) and (iii) are skipped. Step (iv) divides the third interval, and then the GP hyperparameters are updated, resulting in the second diagram ($t = 2, n=3$). At the beginning of iteration $t=3$, it starts conducting\ steps (i)--(v) for the largest intervals. Step (i) selects   the second interval from the left side and step (ii) is skipped. Step (iii) accepts the second interval, because the UCB  within this interval is no less than the center value of the smaller intervals, resulting in the third diagram ($t = 3, n=4$). Iteration $t=3$ continues by conducting steps (i)--(v) for the smaller intervals. Step (i) selects the second interval from the left side, step (ii) accepts it, and step (iii) is skipped, resulting in the forth diagram ($t = 3, n=4$).
The effect of the step (v)  can be seen in the diagrams for iteration $t=9$. At \(n=16\), the far right interval is divided, but no function evaluation occurs. Instead, UCB values given by GP are placed in the new intervals indicated by the red asterisks. One of the temporary dummy values is resolved at \(n=17\) when the interval is queried for division, as shown by the green asterisk. The effect of step (iii) for the rejection case is illustrated in the last diagram for iteration \(t=10\). At \(n=18\), \(t\) is increased to 10 from 9, meaning that  the largest intervals are first considered for division. However, the three largest intervals are all rejected in step (iii), resulting in the division of a very small interval near the global optimum at \(n=18\).

\subsection{Technical Detail of Algorithm}
\begin{algorithm} [t!]
\caption{Infinite-Metric GP Optimization (IMGPO)}
\label{algorithm1}
\begin{algorithmic}[1]
     \fontsize{9.500007pt}{10pt}\selectfont    
    \REQUIRE an objective function $f$, the search domain $\Omega$, the GP kernel $\kappa$, $\Xi_{max} \in \mathbb{N}^+$ and \(\eta \in(0,1)\)
    \vspace{+4pt}
    \STATE Initialize\index{normal} the set $\mathcal{T}_h=\{\emptyset\}\  \forall h\ge0$
    \STATE Set $c_{0,0}$ to be the center point of \(\Omega\) and  \(\mathcal{T}_0\leftarrow\{c_{0,0}\}\)
    \STATE Evaluate \(f\) at \(c_{0,0}\):  $g(c_{0,0}) \leftarrow f(c_{0,0}) $
    \STATE  $f^+\leftarrow g(c_{0,0}),{\cal D} \leftarrow\ \{(c_{0,0},g(c_{0,0}))\}$
    \STATE  $n, N\leftarrow 1, N_{gp} \leftarrow 0,  \Xi\leftarrow 1$
    \FOR {$t = 1,2,3,...$}
       \STATE  $\upsilon_{max}\leftarrow-\infty$
       \FOR[\textcolor{blue}{\textbf{for-loop for steps (i)-(ii)}}]{$h=0$ to $depth(\mathcal{T})$}
          \WHILE{true}
            \STATE{$i_h^* \leftarrow\ \arg \max_{i: c_{h,i} \in \mathcal{T}_h } g(c_{h,i})$}
            \IF{$g(c_{h,i_h^*}) < \upsilon_{max}$}
              \STATE $i_h^* \leftarrow\ \emptyset$, \textbf{break}
            \ELSIF{ $g(c_{h,i_h^*}) $ is \textit{not} labeled as \textit{GP-based}}
              \STATE $\upsilon_{max}\leftarrow g(c_{h,i_h^*})$, \textbf{break}
            \ELSE
              \STATE{$g(c_{h,i_h^*}) \leftarrow f(c_{h,i_h^*})$ and remove  the \textit{GP-based}  label from \(g(c_{h,i_h^*})\)}               \STATE{$N\leftarrow N+1, \; N_{gp}\leftarrow N_{gp}-1$}
              \STATE  ${\cal D} \leftarrow\ \{{\cal D},(c_{h,i_h^*},g(c_{h,i_h^*})) \}$
            \ENDIF
          \ENDWHILE
       \ENDFOR
       \FOR[\textcolor{blue}{\textbf{for-loop for step (iii)}}]{$h=0$ to $depth(\mathcal{T})$}
         \IF{$i_h^* \neq \emptyset$}
           \STATE \small$\xi \leftarrow\ $ the smallest positive integer s.t. \(i_{h+\xi}^* \neq \emptyset\) and \(\xi\le \min(\Xi,\Xi_{max})\) if exists, and 0 otherwise

           \STATE  $z(h,i^*_h)=\max_{k:c_{h+\xi,k} \in \mathcal{T}'_{h+\xi}(c_{h,i^*_{h}})}{\cal U}(c_{h+\xi, k}|{\cal D})$
           \IF{$\xi \neq 0$ and $z(h,i^*_h) < g(c_{h+\xi, i^*_{h+\xi}})$}
             \STATE $i_h^* \leftarrow\ \emptyset$, \textbf{break}
           \ENDIF
         \ENDIF
       \ENDFOR
       \STATE  $\upsilon_{max}\leftarrow-\infty$
       \FOR[\textcolor{blue}{\textbf{for-loop for steps (iv)-(v)}}]{$h=0$ to $depth(\mathcal{T})$}
         \IF{$i_h^* \neq \emptyset$ and $g(c_{h,i_h^*}) \ge \upsilon_{max}$}
         \STATE{$ \; n \leftarrow n+1$}.
             \STATE{\small Divide the hyperrectangle centered at $c_{h,i_h^*}$ along with the longest coordinate into three new hyperrectangles with the following centers: \\   \ \ \    $\mathcal{S} =\{ c_{h+1,i(left)},c_{h+1,i(center)},c_{h+1,i(right)} \} $   }
             \vspace{3pt}
             \STATE{ $\mathcal{T}_{h+1} \leftarrow \{\mathcal{T}_{h+1},\mathcal{S} \}$}
             \STATE $\mathcal{T}_{h} \leftarrow \mathcal{T}_{h} \setminus c_{h,i^*_h}$, $g(c_{h+1,i(center)}) \leftarrow g(c_{h,i^*_h})$
             \FOR{$i_{new}=\{i( left), i(right)\}$}
               \IF{$ {\cal U}(c_{h+1,i_{new}}|{\cal D}) \ge f^+$}
                 \STATE $g(c_{h+1,i_{new}}) \leftarrow f(c_{h+1,i_{new}})$
                 \STATE{ ${\cal D} \leftarrow\ \{{\cal D},(c_{h+1,i_{new}},g(c_{h+1,i_{new}})) \}  $} \\ $N\leftarrow N+1, f^+ \leftarrow \max (f^{+},g(c_{h+1,i_{new}}))$, $\upsilon_{max} = \max(\upsilon_{max},g(c_{h+1,i_{new}}))$
               \ELSE
                 \STATE{$g(c_{h+1,i_{new}}) \leftarrow {\cal U}(c_{h+1,i_{new}}|{\cal D})$ and  label \(g(c_{h+1,i_{new}})\) as \textit{GP-based}. \\ $N_{gp}\leftarrow N_{gp}+1$}
               \ENDIF
             \ENDFOR

         \ENDIF
       \ENDFOR
       \STATE{\small  Update $\Xi$: if $f^+ $ was updated, $\Xi \leftarrow \Xi + 2^{2}$ , and otherwise, $\Xi \leftarrow \max(\Xi - 2^{-1},1)$   }
       \STATE{\small Update GP hyperparameters by an empirical Bayesian method}
    \ENDFOR
\end{algorithmic}
\end{algorithm}

We define \(h\) to be the depth of the hierarchical partitioning tree,  and \(  c_{h,i} \) to be the center  point of  the $i^{th}$ hyperrectangle at depth \(h\). \(N_{gp}\) is the number of the GP evaluations. Define \(depth(\mathcal{T})\) to be the largest  integer \(h\) such that the set \(\mathcal{T}_h \) is not empty. To compute UCB \(\mathcal{U}\), we use \( \varsigma_M=\sqrt{2 \log(\pi^2M^2/12\eta)}\) where \(M\) is the number of the calls made so far for \({\cal U}\) (i.e., each time we use \({\cal U}\), we increment \(M\) by one). This particular form of \(\varsigma_M\) is to maintain the property of  \(f( x) \leq {\cal U} ( x|{\cal D})\) during an execution of our algorithm with probability at least \(1-\eta\). Here, \(\eta\) is the parameter of IMGPO. \(\Xi_{max}\) is another parameter, but it is only used to limit the possibly long computation of step (iii) (in the worst case, step (iii) computes UCBs $3^{\Xi_{max}}$ times although it would rarely happen).

The pseudocode is shown in Algorithm 1. Lines 8 to 23 correspond to steps (i)-(iii).
These lines compute the index \(i_h^*\) of the candidate of the rectangle that may contain a global optimizer for each depth \(h\).
For each depth \(h\), non-null index \(i_{h}^*\) at Line 24 indicates the remaining candidate of a rectangle that we want to divide. Lines 24 to 33 correspond to steps (iv)-(v) where the remaining candidates of the rectangles for all \(h\) are divided. To provide a simple executable division  scheme (line 29), we assume  \(\Omega\) to be a hyperrectangle (see the last paragraph of section
4 for a general case).

Lines 8 to 17  correspond to steps (i)-(ii).  Specifically, line 10 implements step (i) where a single candidate  is selected for each depth, and lines 11 to 12 conduct step (ii) where some candidates are screened out. Lines 13 to 17 resolve the the temporary dummy values computed by GP.
Lines 18 to 23 correspond to step (iii) where the candidates are further screened out. At line 21, $ \mathcal{T}'_{h+\xi}(c_{h,i^*_{h}})$ indicates the set of \textit{all} center points of a fully expanded tree until depth \(h+\xi\) \textit{within} the region covered by the hyperrectangle centered at \(c_{h,i^{*}_h}\). In other words, \(\mathcal{T}'_{h+\xi}(c_{h,i^*_{h}})\) contains the nodes of the fully expanded tree  rooted at \(c_{h,i^{*}_h}\)  with depth \(\xi\) and can be computed by dividing the current rectangle at \(c_{h,i^*_h}\) and recursively divide all the resulting new rectangles until depth \(\xi\) (i.e., depth \(\xi\) from \(c_{h,i^*_h}\), which is depth \(h+\xi\) in the whole tree).

\subsection{Relationship to Previous Algorithms}

The most closely related algorithm is the BaMSOO algorithm \cite{wang2014bayesian}, which combines  SOO with GP-UCB. However, it only achieves a polynomial regret bound while IMGPO achieves a exponential regret bound. IMGPO can achieve exponential regret because it utilizes  the information encoded in the GP prior/posterior to reduce the degree of the unknownness of the semi-metric \(\ell\).

The idea of considering a set of infinitely many bounds   was first proposed by Jones et al. \cite{jones1993lipschitzian}. Their  DIRECT algorithm has been successfully applied to real-world problems \cite{carter2001algorithms, zwolak2005globally}, but it only maintains the consistency property  (i.e., convergence in the limit) from a theoretical viewpoint. DIRECT takes an input parameter \(\epsilon\) to balance the global and local search efforts. This idea was generalized to the case of an unknown semi-metric and strengthened with a  theoretical support (finite regret bound) by Munos \cite{munos2011optimistic} in the SOO algorithm. By limiting the depth of the search tree with a  parameter  \(h_{max}\), the SOO algorithm  achieves a finite regret bound that depends on \textit{the near-optimality dimension}.

\section{Analysis}
In this section, we prove an exponential convergence rate of IMGPO and  theoretically discuss the reason why the novel idea underling IMGPO is beneficial. The proofs are provided in the supplementary material. To examine the effect of considering infinitely many possible candidates of the bounds, we introduce the following term.

\begin{defn}
(Infinite-metric \hspace{-2pt} exploration loss). \hspace{-4pt} The infinite-metric exploration loss  \( \rho_t\) is the number of intervals to be divided during iteration \(t\).
\end{defn}

The infinite-metric exploration loss \(\rho_\tau\) can be computed as $\rho_t = \sum_{h=1}^{depth(\mathcal{T})} \mathbbm{1}(i_{h}^* \neq \emptyset )$ at line 25. It is the cost  (in terms of the number of function evaluations) incurred by not committing to any particular upper bound. If we were to rely on a specific bound, \(\rho_\tau\) would be minimized to 1. For example, the DOO algorithm \cite{munos2011optimistic} has \(\rho_t = 1 \; \forall t \ge1  \).   \textit{Even if we know a particular upper  bound}, relying on this knowledge and thus minimizing \(\rho_\tau\)  is not a good option  \textit{unless the known bound is tight enough compared to the unknown  bound leveraged in our algorithm}. This will be clarified in our analysis. Let \(\bar \rho _{t}\) be the maximum of the averages of \(\rho_{1:t'}\) for \(t'=1,2,...,t\) (i.e., \(\bar \rho_t \equiv \max(\{\frac{1}{t'} \sum_{\tau=1}^{t'}\rho_{\tau} \:; \;t'=1,2, \dots,t \} )\).

\begin{assum}
For some pair of a global optimizer \( x^{*}\) and  an \textit{unknown} semi-metric \(\ell\) that satisfies Assumption 1, both of the following,  (i) shape on \(\ell\) and (ii) lower bound constant,  conditions hold:
\begin{enumerate}[label=(\roman*)]
\item there exist \(L>0\), \(\alpha>0\) and \(p\ge1\) in \(\mathbb{R}\) such that for all \(x,x'\in \Omega\), \(\ell( x',  x)\le L|| x'- x||^\alpha_p\).
\item there exists \(\theta \in (0,1)\) such that for all \(x\in \Omega\), \(f({ x^ * })  \ge  f( x) + \theta \ell \left(x,x^{*}\right)\).
\end{enumerate}
\end{assum}

In Theorem 1, we show that the exponential convergence rate  \(  O \left( \lambda^{N+N_{gp}}  \right) \) with \(\lambda<1\) is achieved.
We define  \(\Xi_{n} \le \Xi_{max} \) to be the largest \(\xi\) used so far with \(n\) total node expansions. For simplicity, we assume that \(\Omega\) is a square, which we satisfied in our experiments by scaling original \(\Omega\).

\begin{theorem}
Assume Assumptions 1 and 2. Let  \(\beta=\sup_{ x,  x'\in \Omega} \frac{1}{2}\| x -  x'\|_\infty \). Let \(\lambda=3^{-\frac{\alpha}{2CD\bar \rho_t}}<1\). Then, with probability at least \(1-\eta\), the regret of IMGPO  is bounded as

\begin{displaymath}
r_N \le L(3\beta D^{1/p})^\alpha \exp \left( - \alpha \left[ \frac{N+N_{gp}}{2CD\bar \rho_{t}}  -\Xi_{n}-2 \right] \ln 3 \right)=O \left( \lambda^{N+N_{gp}}  \right).
\end{displaymath}
\end{theorem}

Importantly, our bound holds for the best values of the unknown \(L,\alpha\) and \(p\) even though these values are not given. The closest result in  previous work is that of BaMSOO \cite{wang2014bayesian}, which obtained \(\tilde O(n^{-\frac{2\alpha}{D(4-\alpha)}})\) with probability \(1-\eta\) for \(\alpha=\{1,2\}\).
As can be seen, we have improved the regret bound. Additionally, in our analysis, we can see how \(L,\)  \(p,\) and \(\alpha\) affect the bound, allowing us to view the inherent difficulty of an objective function in a theoretical perspective. Here, \(C\) is a constant in \(N\) and is used in previous work  \cite{munos2011optimistic,wang2014bayesian}. For example, if we conduct  \(2^D\) or \(3^D-1\) function evaluations per node-expansion and if \(p=\infty\), we have that \(C=1\).

We note that \(\lambda\) can get  close to one as input dimension \(D\) increases, which suggests that there is a remaining challenge in scalability for higher dimensionality. One strategy for addressing  this problem would be to leverage  additional assumptions such as those in   \cite{wang2013bayesian,kandasamy2015high}.

\begin{remark}
(The effect of the tightness of UCB by GP) If UCB computed by GP is ``useful'' such that \(N/\bar \rho_t = \Omega(N)\), then our regret bound becomes \(  O \left( \exp\left(-\frac{N+N_{gp}}{2CD} \alpha \ln3 \right)  \right) \).
If the  bound due to UCB by GP is too loose (and thus useless),  \(\bar \rho_{t}\) can increase up to \(O(N/t)\) (due to \(\bar \rho_t \le \sum^t_{i=1}{i}/{t}\le O(N/t) \)), resulting in the regret bound of \(O \left( \exp\left(-\frac{t(1+N_{gp}/N)}{ 2CD}\alpha \ln3 \right)  \right)\), which can be bounded by \(  O \left( \exp\left(-\frac{N+N_{gp}}{ 2CD} \max(\frac{1}{\sqrt N}, \frac{t}{N})\alpha \ln3 \right)  \right) \)\footnote{\label{footenote}This can be done by limiting the depth of search tree as \(depth(T)=O(\sqrt N)\).  Our proof works with this additional mechanism, but results in the regret  bound with \(N\) being replaced by \(\sqrt N\). Thus, if we assume to have at least ``not useless'' UCBs such that \(N/\bar \rho_t = \Omega(\sqrt{N})\), this additional mechanism can be disadvantageous. Accordingly, we do not adopt it in our experiments.}. This is still better than the known results.
\end{remark}

\begin{remark}
(The effect of GP) Without the use of GP, our regret bound would be as follows: $ r_N \le L(3\beta D^{1/p})^\alpha \exp ( - \alpha [ \frac{N}{2CD} \frac{1}{\tilde \rho_{t}} -2 ] \ln 3 ) $, where \(\bar \rho_{t} \le \tilde \rho_{t}\) is the infinite-metric exploration loss without GP. Therefore, the use of GP reduces the regret bound by increasing \(N_{gp}\) and decreasing \(\bar \rho_{t}\), but may potentially increase the bound by increasing \(\Xi_{n} \le \Xi\).
\end{remark}

\begin{remark}
(The effect of infinite-metric optimization) To understand the effect of considering all the possible upper bounds, we consider the case without GP. If we consider all the possible bounds, we have the regret bound \footnotesize \(L(3\beta D^{1/p})^{\alpha} \exp ( - \alpha [ \frac{N}{2CD} \frac{1}{\tilde \rho_{t}} -2 ] \ln 3 )\) \normalsize for \textit{the best unknown} \(L, \: \alpha\) and \(p\). For standard optimization with a estimated bound, we have \footnotesize \(L'(3\beta D^{1/p'})^{\alpha'} \exp ( - {\alpha'}  [ \frac{N}{2C'D}  -2 ] \ln 3 )\) \normalsize for an estimated \(L',\alpha'\), and \(p'\). By algebraic manipulation, considering all the possible bounds has a better regret when $ \tilde \rho_{t}^{-1} \ge \frac{2CD}{N \ln3^{\alpha}} (( \frac{N}{2C'D}  -2 ) \ln 3^{\alpha'} + 2 \ln 3^\alpha - \ln  \frac{L'(3\beta D^{1/p'})^{\alpha'}}{L(3\beta D^{1/p})^\alpha}   ) $.
For an intuitive insight, we can simplify the above by assuming \(\alpha' = \alpha\) and \(C'=C\) as $
\tilde \rho_t^{-1} \ge 1- \frac{Cc_2D}{N} \ln \frac{L' D^{\alpha/p'}}{LD^{\alpha/p}}  $. Because \(L\) and \(p\) are the ones that achieve the lowest bound, the logarithm on the right-hand side is always non-negative. Hence, \(\tilde \rho_{t} = 1\) always satisfies the condition. When \(L'\) and \(p'\) are not tight enough, the logarithmic term increases in magnitude, allowing \(\tilde \rho_{t} \) to increase. For example, if the second term on the right-hand side has a magnitude of greater than 0.5, then \(\tilde \rho_{t} = 2\) satisfies the inequality. Therefore,  even if we know the   upper bound of the function, we can see that it may be better not to rely on this, but rather take the infinite many possibilities into account.
\end{remark}

\begin{figure*}[b!]
        \centering
        \begin{subfigure}[b]{0.325\linewidth}
                \includegraphics[trim = 20mm 25mm 25mm 40mm, width=\linewidth]{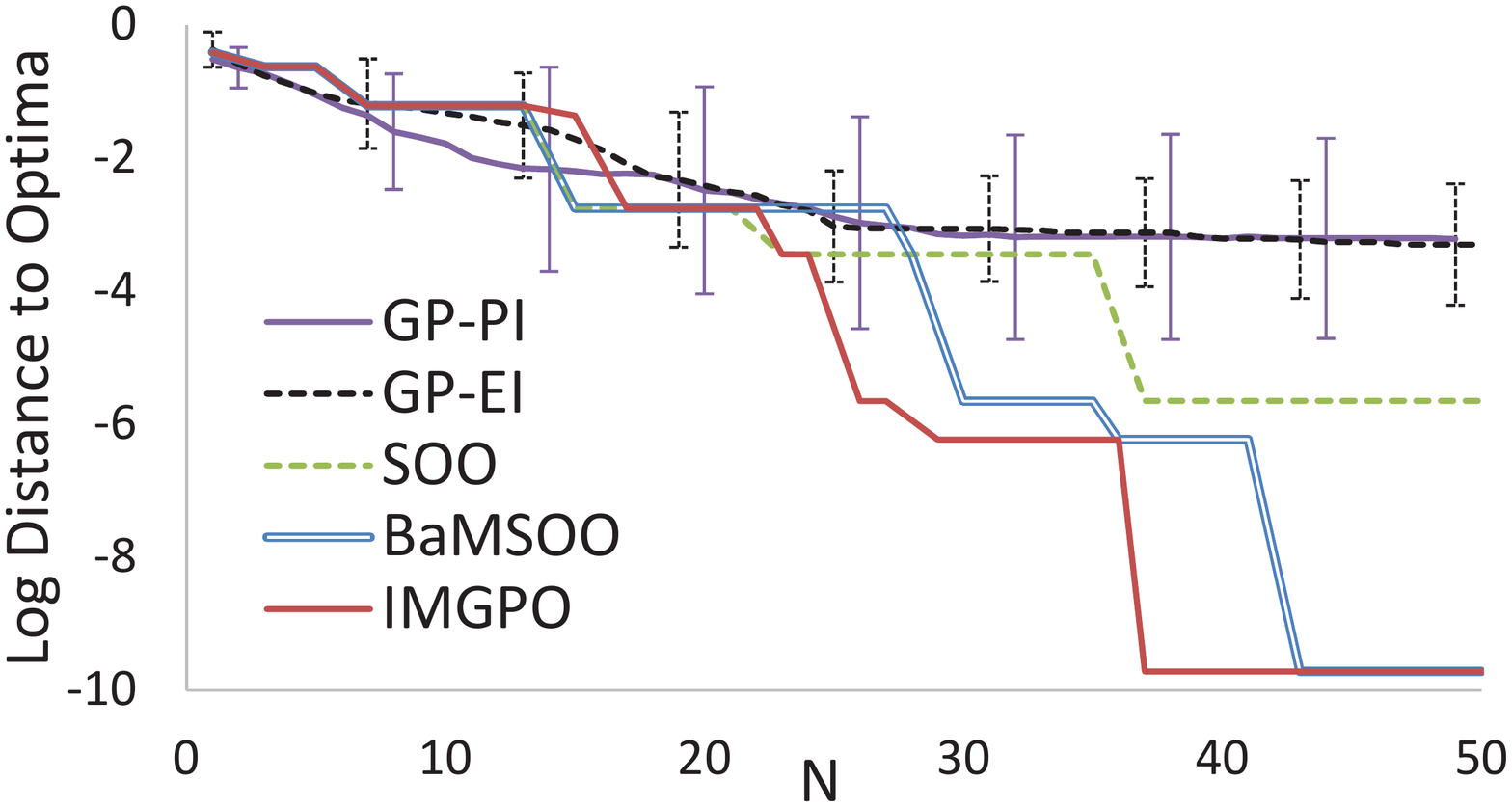}
                \vspace{-25pt}
                \caption{Sin1: [1, 1.92, 2]}
        \end{subfigure}
        \vspace{0pt}
        \begin{subfigure}[b]{0.325\linewidth}
                \includegraphics[trim = 20mm 25mm 24mm 40mm, width=\linewidth]{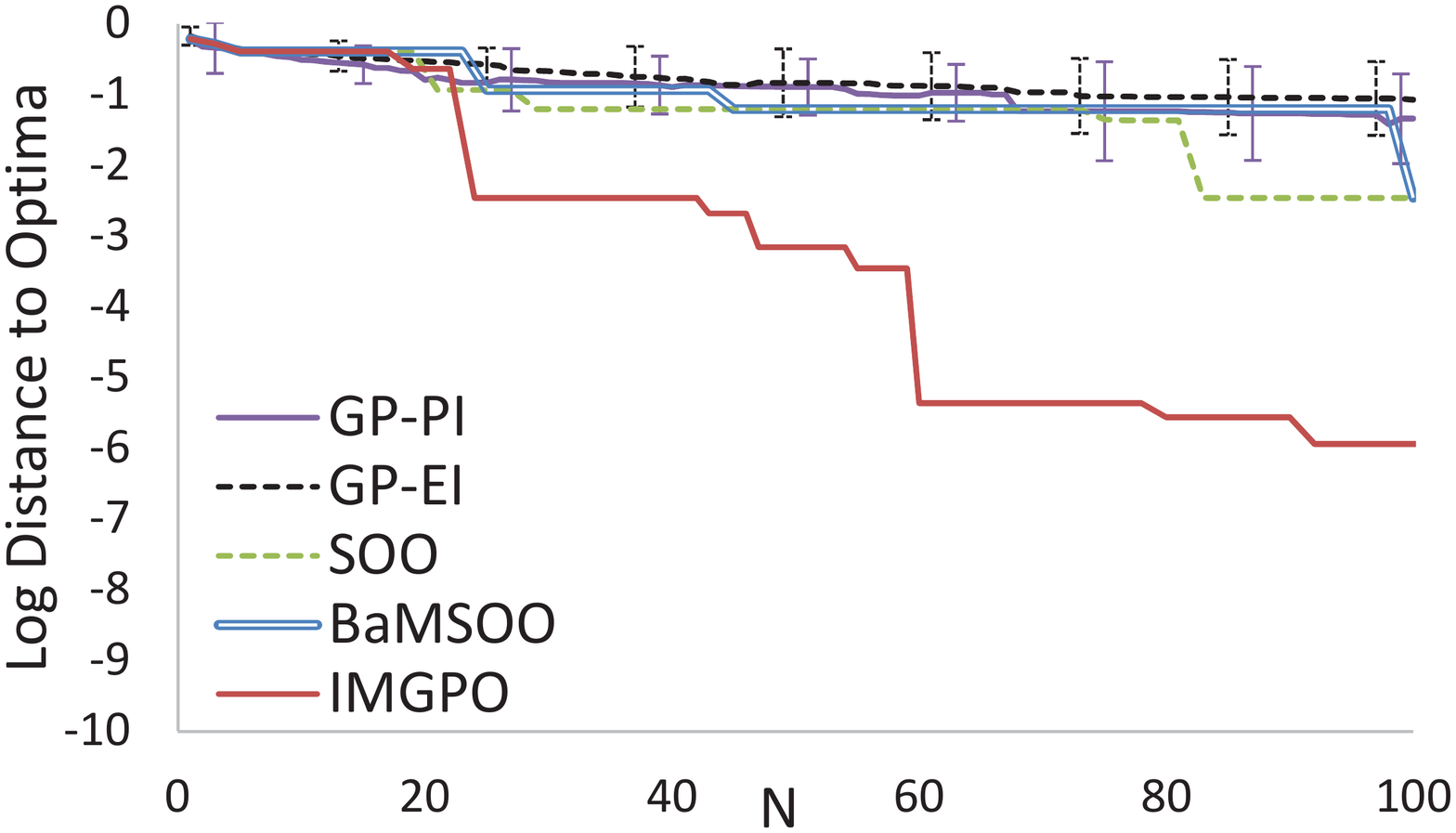}
                \vspace{-25pt}
                \caption{Sin2: [2, 3.37, 3]}
        \end{subfigure}
        \vspace{0pt}
        \begin{subfigure}[b]{0.325\linewidth}
                \includegraphics[trim = 20mm 25mm 23mm 40mm, width=\linewidth]{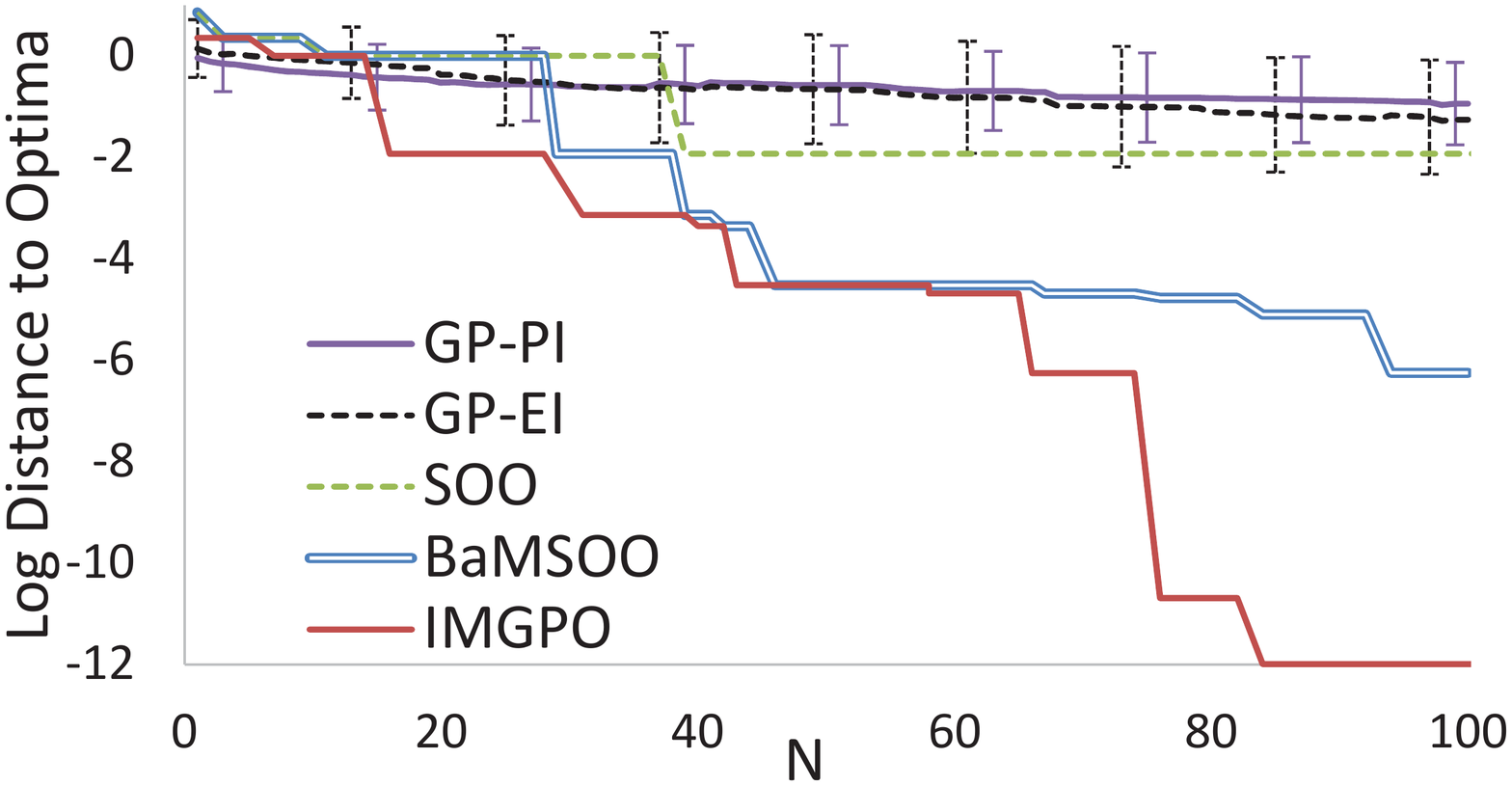}
                \vspace{-25pt}
                \caption{Peaks: [2, 3.14, 4]}
        \end{subfigure}
        \begin{subfigure}[b]{0.325\linewidth}
                \includegraphics[trim = 20mm 25mm 24mm 40mm, width=\linewidth]{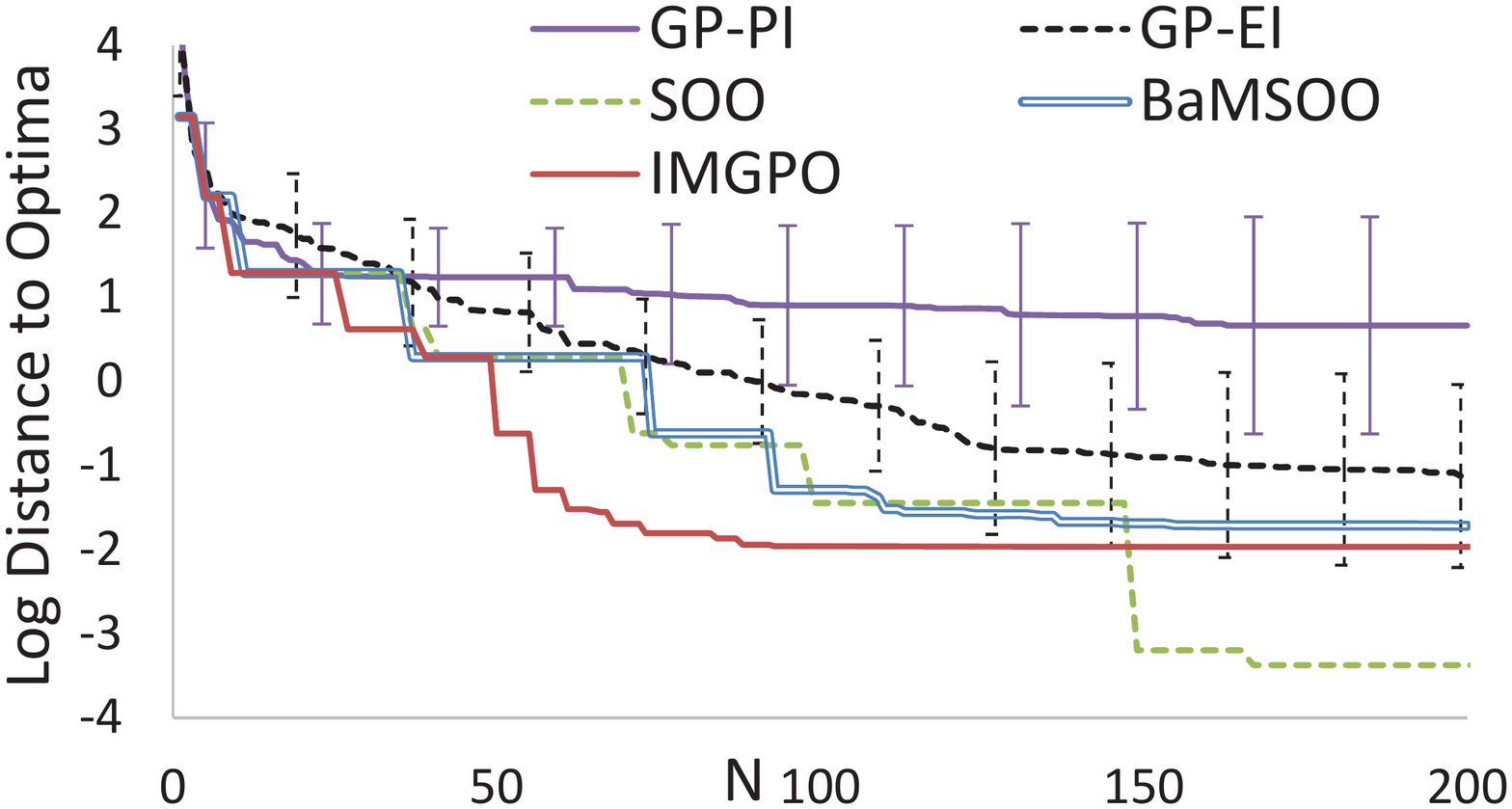}
                \vspace{-25pt}
                \caption{Rosenbrock2: [2, 3.41, 4]}
        \end{subfigure}
        \vspace{0pt}
        \begin{subfigure}[b]{0.325\linewidth}
                \includegraphics[trim = 20mm 25mm 22mm 40mm, width=\linewidth]{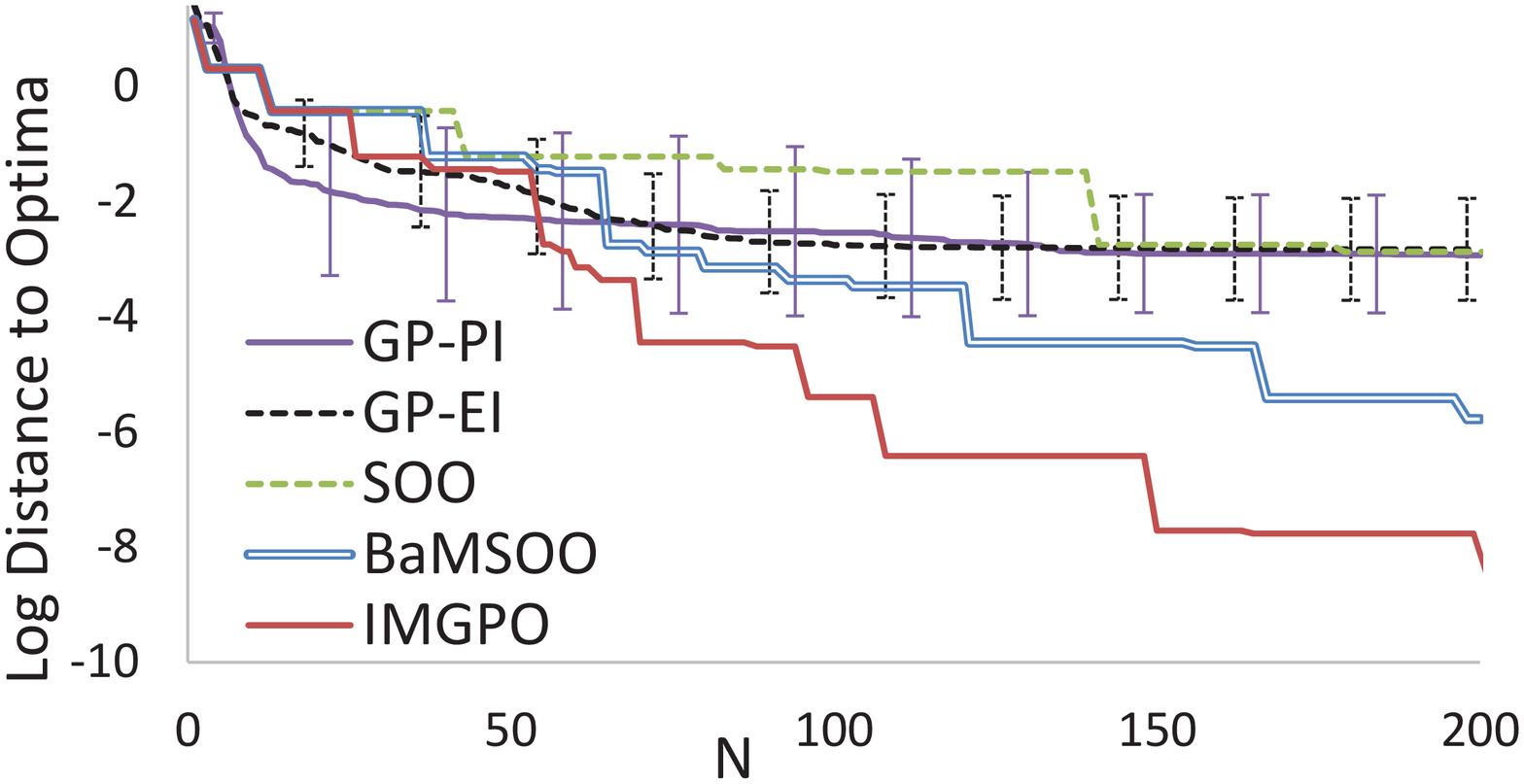}
                \vspace{-25pt}
                \caption{Branin: [2, 4.44, 2]}
        \end{subfigure}
        \begin{subfigure}[b]{0.325\linewidth}
                \includegraphics[trim = 20mm 25mm 22mm 40mm, width=\linewidth]{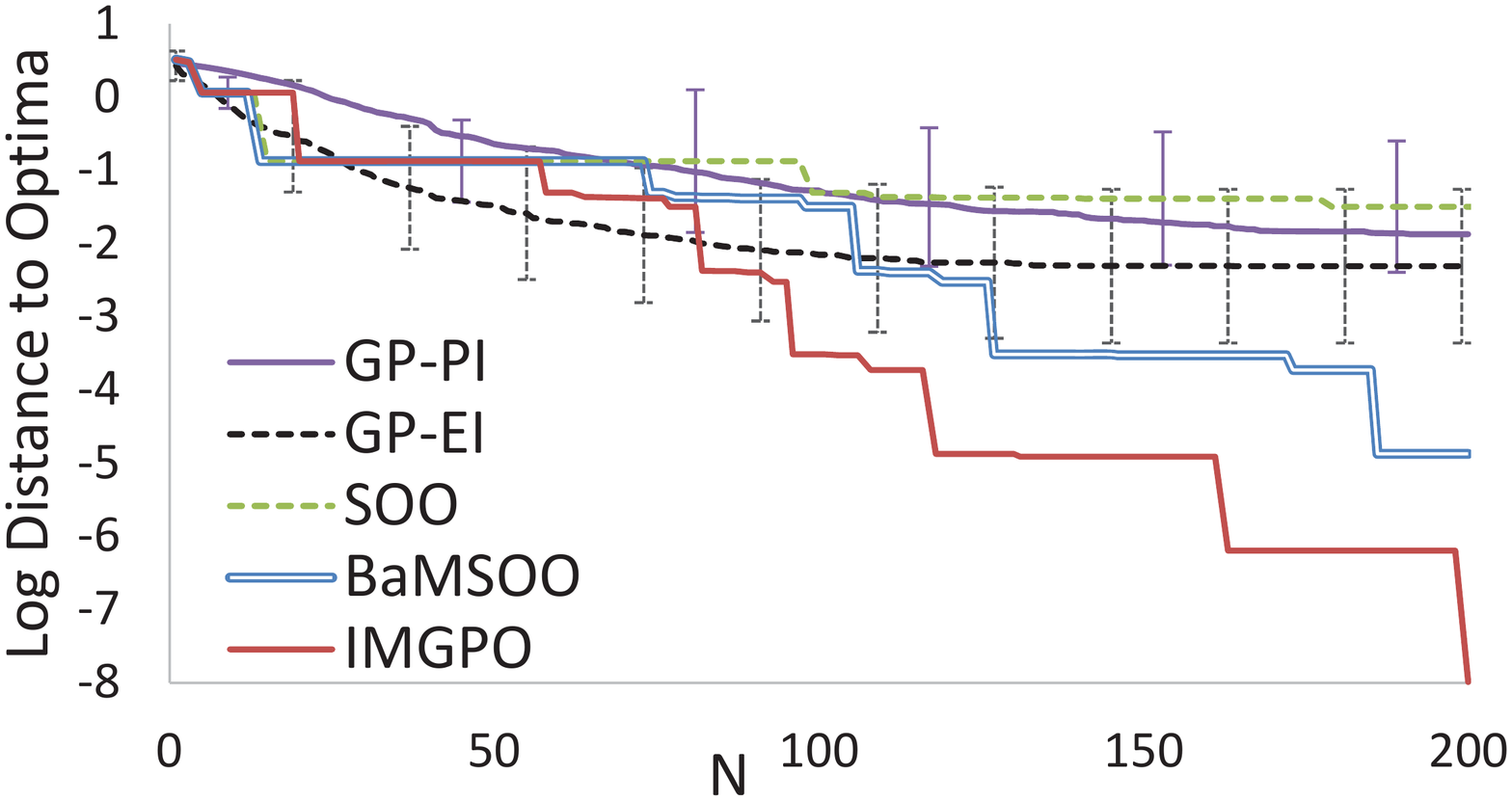}
                \vspace{-25pt}
                \caption{Hartmann3: [3, 4.11, 3]}
        \end{subfigure}
        \vspace{0pt}
        \begin{subfigure}[b]{0.325\linewidth}
                \includegraphics[trim = 20mm 25mm 22mm 40mm, width=\linewidth]{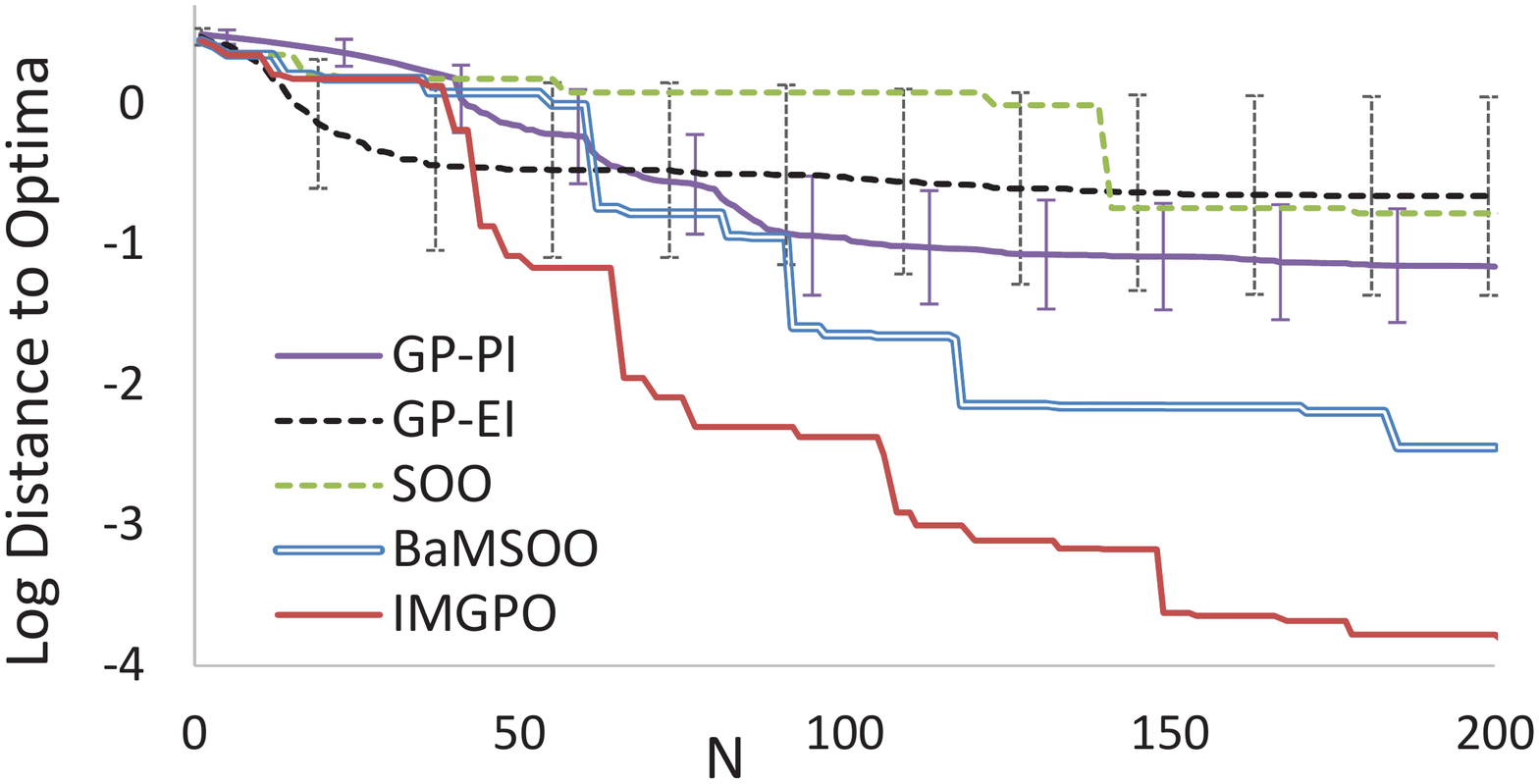}
                \vspace{-25pt}
                \caption{Hartmann6: [6, 4.39, 4]}
        \end{subfigure}
        \vspace{0pt}
        \begin{subfigure}[b]{0.325\linewidth}
                \includegraphics[trim = 20mm 25mm 22mm 40mm, width=\linewidth]{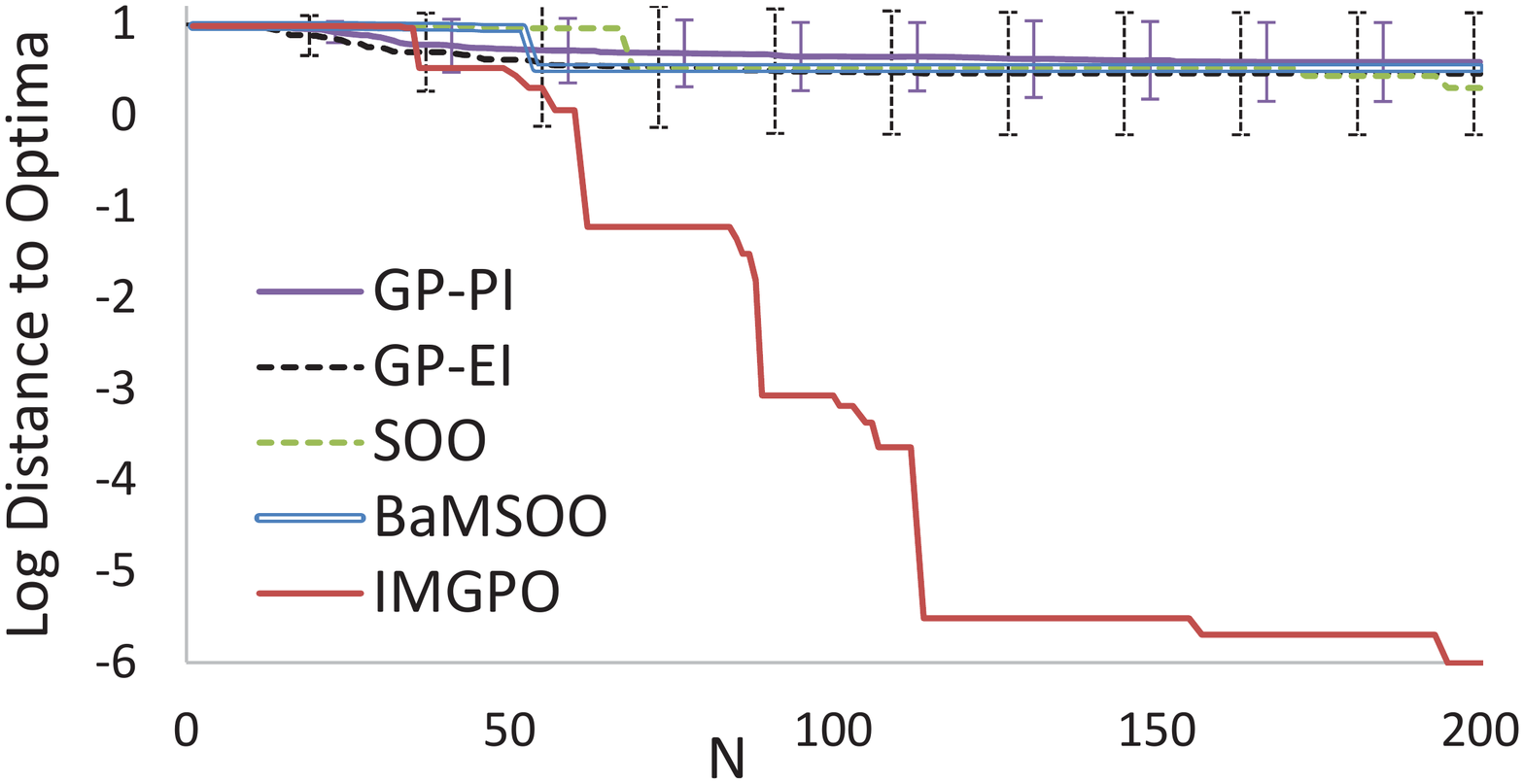}
                \vspace{-25pt}
                \caption{Shekel5: [4, 3.95, 4]}
        \end{subfigure}
        \begin{subfigure}[b]{0.325\linewidth}
                \includegraphics[trim = 20mm 25mm 22mm 40mm, width=\linewidth]{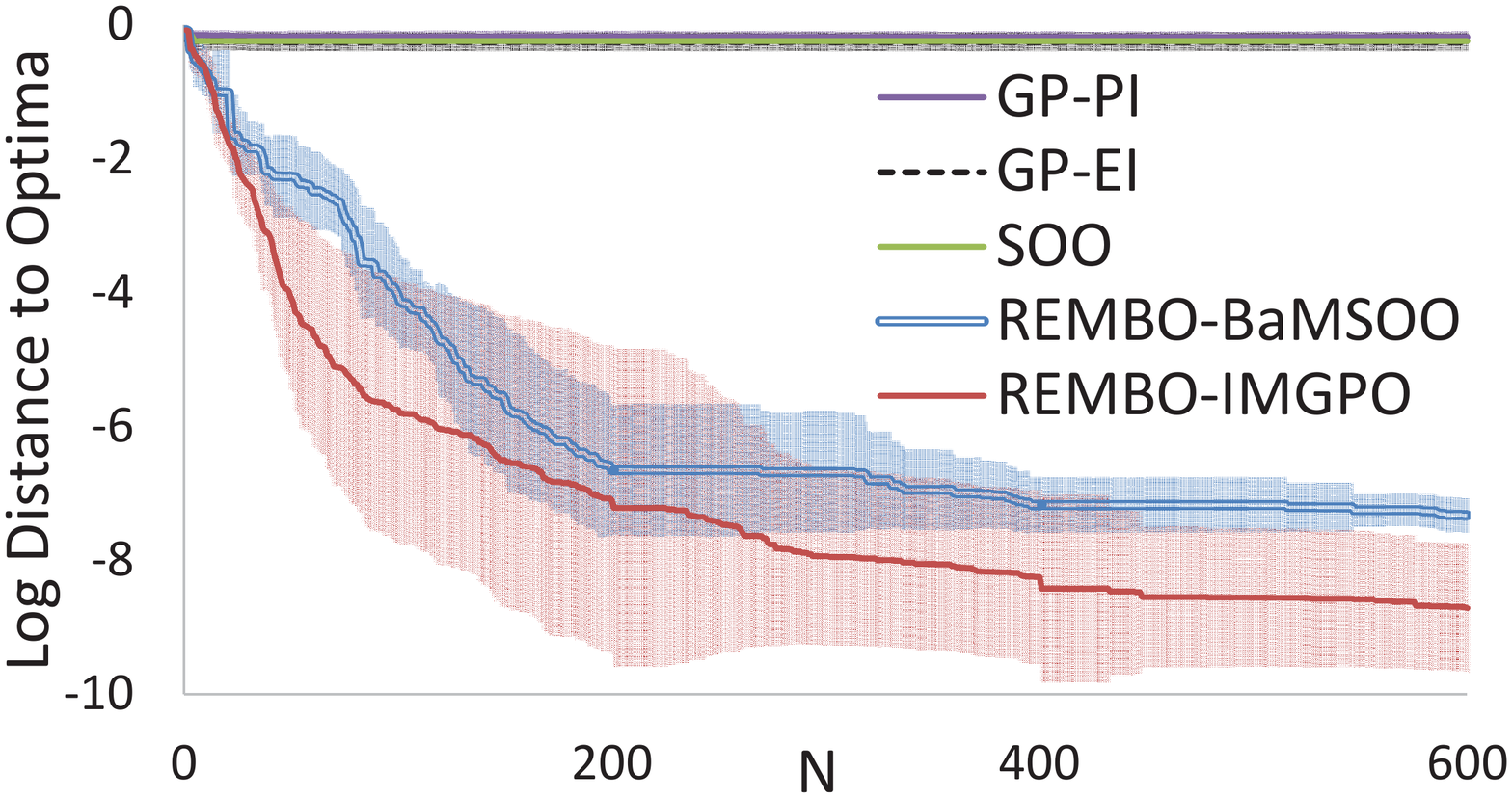}
                \vspace{-25pt}
                \caption{Sin1000: [1000, 3.95, 4]}
        \end{subfigure}
        \caption{Performance Comparison: in the order, the digits inside of the parentheses   [ ] indicate the dimensionality of each function, and the  variables \( \bar\rho_t \) and $\Xi_n$ at the end of computation for IMGPO.  }\label{fig:performance}
\vspace{-10pt}
\end{figure*}

One may improve the algorithm with different division procedures than one presented in Algorithm 1 as discussed in the supplementary material.

\section{Experiments}
In this section, we compare the IMGPO algorithm  with the SOO, BaMSOO, GP-PI and GP-EI algorithms \cite{munos2011optimistic,wang2014bayesian,snoek2012practical}. In previous work, BaMSOO  and GP-UCB were tested with a pair of a handpicked  good kernel and hyperparameters for each  function \cite{wang2014bayesian}. In our experiments, we assume that the knowledge of  good kernel and hyperparameters is unavailable, which is usually the case in practice. Thus, for IMGPO,  BaMSOO,  GP-PI and GP-EI,  we simply used one of the most popular
kernels, the isotropic Matern kernel with \(\nu=5/2\). This is given by $\kappa({x,x'})= g ( \sqrt{5|| {x-x'}||^2/l} )$, where $g(z) = {\sigma ^2}(1 + z + z^2/3)\exp ( - z)$. Then, we blindly initialized the hyperparameters to \(\sigma=1\) and \(l=0.25\) for all the experiments; these values were updated with an empirical Bayesian method after each iteration. To compute the UCB by GP, we used  \(\eta=0.05\) for IMGPO and BaMSOO. For IMGPO, \(\Xi_{max}\) was fixed to be $2^2$ (the effect of selecting different values is discussed later). For BaMSOO and SOO, the parameter \(h_{max}\) was set to \(\sqrt n\), according to Corollary 4.3 in \cite{munos2011optimistic}. For GP-PI and GP-EI, we used the SOO algorithm and a local optimization method using gradients  to solve the auxiliary optimization. For   SOO, BaMSOO and IMGPO, we used the corresponding deterministic division procedure (given \(\Omega\), the initial point is fixed and no randomness exists). For GP-PI and GP-EI, we randomly initialized the first evaluation point and report the mean and one standard deviation for 50 runs.

The experimental results for eight different objective functions  are shown in Figure 2.
The vertical axis is log$_{10}(f( x^*)-f( x^+))$, where \(f( x^*)\) is the global optima and  \(f( x^+)\) is the best value found by the algorithm. Hence, the lower the plotted value on the vertical axis, the better the algorithm's performance. The last five functions are standard benchmarks for global optimization \cite{simulationlib}. The first two were used in \cite{munos2011optimistic} to test SOO, and can be written as \(f_{sin1}(x)=(\sin(13x)\sin+1)/2\) for Sin1 and \(f_{sin2}( x)=f_{sin1}(x_1)f_{sin1}(x_2)\) for Sin2. The form of the third function is given in Equation (16) and Figure 2 in \cite{mcdonald2007global}. The last function is Sin2 embedded in 1000 dimension in the same manner described in Section 4.1 in \cite{wang2013bayesian}, which is used here to illustrate a possibility of using IMGPO as a main subroutine to scale up to higher dimensions with additional assumptions. For this function, we used REMBO \cite{wang2013bayesian} with IMGPO and BaMSOO as its Bayesian optimization subroutine. All of these functions are multimodal, except for Rosenbrock2, with dimensionality from 1 to 1000.

\begin{table}[t!] \footnotesize
\caption{Average CPU time  (in seconds) for the experiment with each test function}
\vspace{-8pt}
\begin{center}
\begin{tabular}{ l | c | c | c |c|c|c|c|c}
\hline
\textbf{Algorithm} & \textbf Sin1 & Sin2 & Peaks & Rosenbrock2 & Branin & Hartmann3 & Hartmann6 & Shekel5 \\ \hline
GP-PI & 29.66 & 115.90 & 47.90 & 921.82 & 1124.21 & 573.67&657.36 & 611.01\\
GP-EI & 12.74 & 115.79 & 44.94 &893.04 & 1153.49 &562.08 &604.93& 558.58\\
SOO & 0.19 & 0.19 & 0.24  & 0.744& 0.33& 0.30& 0.25& 0.29\\
BaMSOO & 43.80 & 4.61 & 7.83  & 12.09 &14.86 &14.14 &26.68&371.36\\
IMGPO & 1.61 & 3.15 & 4.70 & 11.11 & 5.73&6.80 &13.47 & 15.92\\ \hline
\end{tabular}
\end{center}
\label{aaa}
\vspace{-9pt}
\end{table}
As we can  see from Figure 2, IMGPO outperformed the other algorithms in general. SOO produced the competitive results for Rosenbrock2 because our GP prior was misleading (i.e., it did not model the objective function well and thus the property \(f( x) \leq {\cal U} ( x|{\cal D})\) did not hold many times). As can be seen in Table 1, IMGPO is much faster than traditional GP optimization methods although it is slower than SOO. For Sin 1, Sin2, Branin and Hartmann3, increasing $\Xi_{max}$ does not affect IMGPO because $\Xi_{n}$ did not reach $\Xi_{max}=2^2$ (Figure 2). For the rest of the test functions, we would be able to improve the performance of IMGPO by increasing $\Xi_{max}$ at the cost of extra CPU time.

\section{Conclusion}
We have presented the first   GP-based optimization method with an exponential convergence rate \(O \left( \lambda^{N+N_{gp}}  \right)\)  ($\lambda<1$) \textit{without} the need of  auxiliary optimization and the \(\delta\)-cover sampling. Perhaps more importantly in the viewpoint of a broader global optimization community, we have provided a  practically oriented analysis framework, enabling us to see why \textit{not} relying on a particular bound is advantageous, and how a non-tight bound can still be useful (in Remarks 1, 2 and 3). Following the advent of the DIRECT algorithm, the literature diverged along two paths, one with a particular bound  and one without. GP-UCB can be categorized into the former. Our approach illustrates the benefits of combining  these two paths.

As stated in Section 3.1, our solution idea was  to use  a bound-based method but rely less on the estimated bound by considering all the possible bounds. It would be interesting to see if a similar principle can be applicable to other types of bound-based methods such as planning algorithms (e.g., A* search and the UCT or FSSS algorithm \cite{walsh2010integrating}) and learning algorithms (e.g., PAC-MDP algorithms \cite{strehl2009reinforcement}).

\subsubsection*{Acknowledgments}
The authors would like to thank Dr. Remi Munos for his thoughtful comments and suggestions.
We gratefully acknowledge support from NSF grant 1420927, from ONR grant N00014-14-1-0486, and from ARO grant W911NF1410433. Kenji Kawaguchi was supported in part by the Funai Overseas Scholarship. Any opinions, findings, and conclusions or recommendations expressed in this material are those of the authors and do not necessarily reflect the views of our sponsors.

\newpage
\appendix
\renewcommand{\thesubsection}{\Alph{subsection}}

\vbox{\hsize\textwidth
\linewidth\hsize \vskip 0.1in \toptitlebar \centering
{\LARGE\bf Bayesian Optimization with Exponential Convergence: Supplementary Material\par}
\bottomtitlebar 
\vskip 0.3in minus 0.1in
}


In this supplementary material, we provide the proofs of the theoretical results. Along the way, we  also prove regret bounds for a general class of algorithms, the result of which may be used to design a new algorithm.

We first provide a known property of the upper confidence bound of GP.
\begin{lemma}
(Bound Estimated by GP)  According to the belief encoded in the GP prior/posterior\footnote{Thus, the probability  in this analysis should be seen as  that of \textit{the subjective view}. If we assume that  \(f\) is indeed a sample from the GP, we have the same result  with \textit{the objective view}  of probability.   }, for any \( x\), $f( x) \leq {\cal U} ( x|{\cal D})$  holds during the execution of Algorithm 1   with probability at least \(1-\eta\).
\end{lemma}

\begin{proof}
It follows the proof of lemma 5.1 of \cite{srinivas2010gaussian}.
From the property of the standard gaussian distribution, \(\Pr(f( x) > {\cal U} ( x|{\cal D}))<\frac{1}{2}e^{-\varsigma^{2}_M/2}\). Taking union bound on the entire execution of Algorithm 1, \(\Pr(f( x) > {\cal U} ( x|{\cal D}) \; \forall M\ge 1)<\frac{1}{2}\sum_{M=1}^{\infty} e^{-\varsigma^{2}_M/2}\). Substituting \(\varsigma_{M} = \sqrt{2 \log(\pi^{2}M^{2}/12\eta)}\), we obtain the statement.
\end{proof}

Our algorithm has a concrete division procedure in line 27 of Algorithm 1. However, one may improve the algorithm with different division procedures. Accordingly, we first derive abstract version of regret bound for the IMGPO (Algorithm 1) under a family of division procedures  that satisfy  Assumptions 3 and 4. After that, we provide a proof for the main results in the paper.

\subsection{With Family of Division Procedure}
In this section, we modify the result obtained by \cite{munos2011optimistic}. Let \(x_{h,i}\) to be any point in the region covered by the \(i\)\textsuperscript{th} hyperinterval at depth \(h\), and \(x_{h,i}^{*}\) be the global optimizer that may exist in the \(i\)\textsuperscript{th} hyperinterval at depth \(h\). The previous work provided the regret bound of the SOO algorithm with a family of division procedure that satisfies the following two assumptions.

\setcounter{assum}{2}
\begin{assum}
(Decreasing diameter) There exists a diameter function \(\delta (h) > 0\) such that, for any hyperinterval \({ \omega_{h,i}} \subset \Omega\) and its center \(c_{h,i} \in \omega_{h,i} \) and any  \(x_{h,i} \in \omega_{h,i}\), we have \(\delta (h) \ge {\sup _{{x_{h,i}}}}\ell ({x_{h,i}},{c_{h,i}})\) and \(\delta (h - 1) \ge \delta (h)\) for all \(h\ge\)1.
\end{assum}

\begin{assum}
(Well-shaped cell) There exists  \(\nu >\) 0  such that any hyperinterval \({\omega _{h,i}}\) contains at least an \(\ell \)-ball of radius \(\nu \delta (h)\) centered in \({\omega _{h,i}}\).
\end{assum}
Thus, in this section, hyperinterval is not restricted to hyperrectangle. We now revisit the definitions of several terms and variables used in \cite{munos2011optimistic}. Let the \(\epsilon\)-optimal space \(X_\epsilon\) be defined as \(X_\epsilon: = \{ x \in \Omega :f(x) + \epsilon \ge f({x^ * })\} \). That is, the \(\epsilon\)-optimal space  is the set of input vectors whose function value is at least \(\epsilon\)-close to  the global optima. To bound the number of hyperintervals relevant to this \(\epsilon\)-optimal space, we define a near-optimality dimension as follows.

\setcounter{defn}{2}
\begin{defn}
(Near-optimality dimension) The near-optimality dimension is the smallest \(d>0\) such that, there exists \(C>0,\) for all \(\epsilon > 0,\) the maximum number of disjoint \(\ell \)-balls of radius \(\nu\epsilon\) with center in the \(\epsilon\)-optimal space \(X_\epsilon\) is less than  \(C \epsilon^{-d}\).
\end{defn}

Finally, we define the set of \(\delta \)-optimal hyperintervals \(I_{\delta(h)}\) as \({I_{\delta(h)}}: = \{c \in \Omega: f(c_{}) + \delta (h) \ge f({x^ * }), c \text{ is the center point of the interval}, \omega_{h,i}, \text{for some } (h,i)\} \).
The \(\delta \)-optimal hyperinterval \({I_{\delta(h)}}\) is used to relate the hyperintervals to the \(\epsilon\)-optimal space. Indeed, the \(\delta \)-optimal hyperinterval \({I_{\delta(h)}}\) is almost identical to the \(\delta (h)\)-optimal space \(X_{\delta (h)}\), except that \({I_{\delta(h)}}\) is focused on the center points whereas \({X_{\delta (h )}}\) considers the whole input vector space. In the following, we use \(|{I_{\delta (h )}}|\) to denote the number of \({I_{\delta (h )} }\) and derive its upper bound.

\begin{lemma}
(Lemma 3.1 in \cite{munos2011optimistic}) Let \(d\) be the near-optimality dimension and \(C\) denote the corresponding constant in Definition 1.  Then, the number of \(\delta \)-optimal hyperintervals is bounded by
\(|{I_{\delta(h)}}| \le C\delta {(h)^{ - d}}\).
\end{lemma}
We are now ready to present the main result in this section. In the following, we use the term \textit{optimal hyperinterval} to indicate a hyperinterval that contains a global optimizer \( x^*\). We say a hyperinterval is \textit{dominated} by other intervals when it is rejected or not selected in step (i)-(iii). In Lemma 3, we bound the maximum size of the optimal hyperinterval. From Assumption 1, this  can be translated to the regret bound, as we shall see in Theorem 2.

\begin{lemma}
Let \(\Xi_{n} \le \min(\Xi,\Xi_{max}) \) be the largest \(\xi\) used so far with \(n\) total node expansions. Let \(h_n^ * \) be the depth of the deepest expanded node that contains a global optimizer \({x^ * }\) after \(n\) total node expansions (i.e., \(h_n^* \le n\) determines the size of the \textit{optimal hyperinterval}). Then, with probability at least \(1-\eta\), \(h_n^ * \) is bounded below by some \( h' \)  that satisfies
\begin{equation*}
n \ge   \sum_{\tau\ =1}^{\sum_{l=0} ^{h'+ \Xi} |I_l|} \rho_\tau.
\end{equation*}
\end{lemma}

\begin{proof}
Let \(T_h\) denote the time at which the optimal hyperinterval is further divided. We prove the statement by showing that the time difference \(T_{h+1}-T_h\) is  bounded by the number of \(\delta\)-optimal hyperintervals.  To do so, we first note that there are three types of hyperinterval that can dominate an optimal hyperinterval  \(c_{h+1,*}\) during the time \([T_h,T_{h+1}-1]\), all of which belong to  \(\delta\)-optimal hyperintervals \(I_{\delta}\). The first type  has the same size (i.e., same depth \(h\)), \(c_{h+1,i}\). In this case,
\[f(c_{h+1,i}) \ge f(c_{h+1,*})  \ge f({x_{h+1,*}^*})- \delta (h+1),\]
where the first inequality is due to line 10 (step (i)) and the second follows Assumptions 1 and  2. Thus, it must be \(c_{h+1,i} \in I_{h+1} \). The second case is where the optimal hyperinterval may be dominated by a hyperinterval of larger size (depth \(l < h+1\)), \(c_{l,i}\). In this case, similarly,
\[f(c_{l,i}) \ge f(c_{h+1,*}) \ge f({x_{h+1,*}^*})-\delta (l),\]
where the first inequality is due to lines 11 to 12 (step (ii)) and thus \(c_{l,i} \in I_l \). In the final scenario, the optimal hyperinterval is dominated by a hyperinterval of smaller size (depth \(h+1+\xi \)), \(c_{h+1+\xi,i}\). In this case,
\[f(c_{h+1+\xi,i}) \ge z(h+1,*)\ge f({x_{h+1,*}^*}) - \delta (h+1+\xi) \]
with probability at least \(1-\eta\) where \(z(\cdot,\cdot)\) is defined in line 21 of Algorithm 1. The first inequality is due to lines 19 to 23 (step (iii))  and the second inequality follows Lemma 1 and Assumptions 1 and 3. Hence, we can see that \(c_{h+1+\xi,i} \in I_{h+1+\xi}\).

For all of the above arguments, the temporarily assigned \({\cal U}\) under GP has no effect. This is because   the algorithm still covers the above three types of \(\delta\)-optimal hyperintervals \(I_{\delta}\), as \({\cal U} \ge f\) with probability at least \(1-\eta\) (Lemma 1). However, these are only expanded  based on \(f\) because of the temporary nature of  \({\cal U}\). Putting these results together,
\[T_{h+1}-T_h \le \sum_{\tau=1}^{\sum_{l=1} ^{h+1+\Xi_{n}} |I_{\delta(l)}|} \rho_\tau.\]
Since if one of the \(I_\delta\) is divided during \([T_h,T_{h+1}-1]\),  it cannot be divided again during another time period,
\[\sum_{h=0}^{h_n^*} T_{h+1}-T_h \le \sum_{\tau=1}^{\sum_{l=1} ^{h_n^*+1+\Xi_{n}} |I_l|} \rho_\tau,\]
where on the right-hand side, we could combine the summation \(\sum^{h_n^*}_{h=0} \) and \(\sum_{\tau=1}^{\sum_{l=1} ^{h+1+\Xi_{n}} |I_{\delta(l)}|}\) into the one, because each \(h\) in the summation refers to the same \(\delta\)-optimal interval \(I_{\delta(l)}\) with \(l \le h_n^*+1 +\Xi_{n}\), and should not be double-counted.  As \(\sum_{h=0}^{h_n^*} T_{h+1}-T_h =T_{h_n^*+1} -T_0\), \(T_0=1\) and \(|I_{\delta(0)}|=1\),
 \[T_{h_n^*+1} \le 1+ \sum_{\tau=1}^{\sum_{l=1} ^{h_n^*+1+\Xi_{n}} |I_l|} \rho_\tau  \le \sum_{\tau=1}^{\sum_{l=0} ^{h_n^*+1+\Xi_{n}} |I_l|} \rho_\tau.\]
As \(T_{{h_n^*}+1} > n\) by definition, for any \(h'\) such that $
\sum_{\tau=1}^{\sum_{l=0} ^{h'+\Xi_{n}} |I_l|} \rho_\tau \le n <  \sum_{\tau=1}^{\sum_{l=0} ^{h_n^*+1+\Xi_{n}} |I_l|} \rho_\tau $, we have \(h_n^*>h'\).
\end{proof}

With Lemmas 2 and 3, we are ready to present a finite regret bound with the family of division procedures.

\setcounter{theorem}{1}
\begin{theorem}
Assume Assumptions 1, 3, and 4. Let \(h(n)\) be the smallest integer \(h\) such that
\begin{displaymath}
n \le  \sum_{\tau=1}^{C \sum_{l=0} ^{h+\Xi_{n}} \delta (l)^{-d}} \rho_\tau.
\end{displaymath}
Then, with probability at least \(1-\eta\), the regret of the IMGPO with any general division procedure is bounded as
\vspace{-2pt}
\begin{displaymath}
r_n \le \delta(h(n)-1).
\end{displaymath}
\end{theorem}

\begin{proof}
Let \(c(n)\) and \(c_{h_{n}^{*},*}\) be the center point expanded at the \(n\)th expansion and the optimal hyperinterval containing a global optimizer \(x^{*}\), respectively. Then, from Assumptions 1, 3, and 4, \(f(c(n)) \ge f(c_{h_n^*,*}) \ge f^* - \delta(h_n^*) \), where \(f^{*}\) is the global optima. Hence, the regret bound is \(r_{h} \le \delta(h_n^*)\). To find a lower  bound for the quantity \(h_n^*\), we first relate  \(h(n)\) to Lemma 3 by
\[
n > \sum_{\tau=1}^{C \sum_{l=0} ^{h(n)+\Xi_{n}-1} \delta (l)^{-d}} \rho_\tau\ \ge  \sum_{\tau=1}^{\sum_{l=0} ^{h(n)+ \Xi_{n}-1} |I_l|} \rho_\tau,
\]
 where the first inequality comes from the definition of \(h(n)\), and the second follows from Lemma 2. Then, from Lemma 3, we have \(h_n^* \ge h(n)-1 \). Therefore, \( r_n \le \delta(h_n^*) \le \delta(h(n)-1) \).
\end{proof}

\begin{assum}
(Decreasing diameter revisit) The decreasing diameter defined in Assumption 3 can be written as  \(\delta(h)=c_{1}\gamma^{h/D}\) for some \(c_1>0\) and \(\gamma<1\) with a division procedure that requires \(c_{2}\)  function evaluations  per  node expansion.
\end{assum}

\begin{corollary}
Assume Assumptions 1, 3, 4, and 5. Then, if \(d=0\),
with probability at least \(1-\eta\),
\begin{displaymath}
r_N \le O \left( \exp\left(-\frac{N+N_{gp}}{ c_{2}CD\bar \rho_{t}} \right)  \right).
\end{displaymath}
If \(d>0\),
with probability at least \(1-\eta\),\begin{displaymath}
r_N \le O \left( \left( \frac{1}{N+N_{gp}} \right)^{1/d} \left(-\frac{c_2C\bar \rho_{t}}{1-\gamma^{d/D}} \right)^{1/d} \gamma^{-\frac{1}{D}} \right).
\end{displaymath}
\end{corollary}

\begin{proof}
For the case \(d=0\), we have $ n \le  \sum_{\tau=1}^{C \sum_{l=0} ^{h(n)+\Xi_{n}} \delta (l)^{-d}} \rho_\tau \le \sum_{\tau=1}^{C(h(n)+\Xi_{n}+1) } \bar \rho_t $, where the first inequality follows from the definition of \(h(n)\), and the second  comes from the definition of \(\bar \rho_t\) and  the assumption \(d=0\). The second inequality holds for \(\bar \rho_{t}\) that only considers \(\rho_\tau\) with \(\tau \le t\). This is computable, because  \(\tau \le t \) by  construction. Indeed, the condition of Lemma 3 implies \(t \ge \sum_{l=0} ^{h'+\Xi_{n}} |I_l|\). Therefore, the two inequalities hold, and we can  deduce that \(h(n) \ge \frac{n}{C\bar \rho_{t}} - \Xi_{n}-1\) by algebraic manipulation. By Assumption 5,  \(n=(N+N_{gp})/c_2\). With this, substituting the lower bound of \(h(n)\) into the statement of Theorem 2 with Assumption 5,
\[
 r_N \le c_1\exp \left( - \left[ \frac{N+N_{gp}}{c_{2}D} \frac{1}{C\bar \rho_{t}} -\Xi_{n}-2 \right] \ln \frac{1}{\gamma} \right). \] Similarly, for the case \(d>0\),
\[ n \le  \sum_{\tau=1}^{C \sum_{l=0} ^{h(n)+\Xi_{n}} \delta (l)^{-d}} \rho_\tau \le \sum_{\tau=1}^{c^{-d}C \frac{\gamma^{-(h(n)+\Xi_{n}+1)d/D}-1 }{\gamma^{-d/D}-1} } \bar \rho_t, \]
and hence \(c\gamma^{\frac {h(n)+\Xi_{n}}{D}} \le \left(\frac{n(1-\gamma^{d/D})}{C \bar \rho_t} \right)^{-1/d} \) by algebraic manipulation. Substituting this into the result of Theorem 2, we arrive at the desired result.
\end{proof}

\subsection{With a Concrete Division Procedure}
In this section, we prove the main result in the paper. In Theorem 1, we show that the exponential convergence rate bound \(O \left( \lambda^{N + N_{gp}} \right)\) with \(\lambda < 1\) is achieved  \textit{without} Assumptions 3, 4 and 5 and \textit{without} the assumption that \(d=0\).
\setcounter{theorem}{0}
\begin{theorem}
Assume Assumptions 1 and 2. Let  \(\beta=\sup_{ x,  x'\in \Omega}\frac{1}{2}\| x -  x'\|_\infty \). Let \(\lambda=3^{-\frac{\alpha}{2C\bar \rho_t D}}<1\). Then, \textit{without} Assumptions 3, 4 and 5 and \textit{without} the assumption on \(d\),  with probability at least \(1-\eta\), the regret of IMGPO with the division procedure in Algorithm 1 is bounded as
\begin{displaymath}
r_N \le L(3\beta D^{1/p})^\alpha \exp \left( - \alpha \left[ \frac{N+N_{gp}}{2C\bar \rho_{t}D}  -\Xi_{n}-2 \right] \ln 3 \right)= O \left( \lambda^{N+N_{gp}}  \right).
\end{displaymath}
\end{theorem}

\begin{proof}
To prove the statement, we show that Assumptions 3, 4, and 5 can all be satisfied while maintaining  \(d=0\). \vspace{-2pt}

From Assumption 2 (i), and based on the division procedure that the algorithm uses,
\[\sup_{x \in \omega_{h,i}} \ell(x,c_{h,i}) \le \sup_{x \in \omega_{h,i}}L|| x- c_{h,i}||^{\alpha}_{p} \le L \left(3^{-\floor{h/D}} \beta D^{1/p} \right)^\alpha. \]
This upper bound corresponds to the diagonal length of each hyperrectangle with respect to \(p\)-norm, where \(3^{- \floor{h/D}}\beta\) corresponds to the length of the longest side. We fix the form of \(\delta\) as \(\delta(h)=L3^\alpha D^{\alpha/p}3^{-h\alpha/D} \beta^\alpha \ge L ( 3^{-\floor{h/D}} \beta D^{1/p} )^\alpha\), which  satisfies Assumption 3.

This form of \(\delta(h)\) also satisfies Assumption 5 with \(\gamma=3^{-\alpha}\) and \(c_1=L3^\alpha D^{\alpha/p}\beta^\alpha\).

Every hyperrectangle contains at least one \(\ell\)-ball with a radius corresponding to the length of the shortest side of the hyperrectangle. Thus, we have at least one \(\ell\)-ball of radius \(\nu \delta(h)=L3^{-\alpha \ceil{h/D}} \ge L3^{-\alpha}3^{-\alpha h/D}\) for every hyperrectangle with \(\nu \ge 3^{-2\alpha}D^{-\alpha/p}\). This satisfies Assumption 4.

Finally, we  show that \(d=0\). The set of \(\delta \)-optimal hyperintervals \(I_{\delta(h)}\) is contained by the \(\delta(h)\)-optimal space \(X_{\delta(h)}\) as
\begin{align*}
{I_{\delta(h)}} &= \{c \in \Omega:  f({x^ * }) -f(c_{})  \le \delta (h), c \text{ is the center point of the interval},  \omega_{h,i}, \text{for some } (h,i)\}
\\ & \subseteq  \{ x \in \Omega :f({x^ * }) -f(x_{})  \le \delta (h)\} = X_{\delta(h)}
\end{align*}

Let \(\theta\) be a value that satisfies Assumption 2 (ii) (which is nonzero). Consider an \(\ell\)-ball of radius \(\frac{\delta(h)}{\theta}\) at \(x^*\), which is a set \(\{x \in \Omega \ |\ \theta \ell(x,x^*) \le \delta(h)\}\). Since \(\theta \ell(x,x^*)\le f(x^*) -f(x)\) by Assumption 2 (ii),  the \(\delta(h)\)-optimal space \(X_{\delta(h)}\) is covered by an \(\ell\)-ball of radius \(\frac{\delta(h)}{\theta}\). Therefore, \({I_{\delta(h)}}\subseteq X_{\delta(h)} \subseteq \text{(an \(\ell\)-ball of radius \(\frac{\delta(h)}{\theta}\) at \(x^*\))}\). By Assumption 2 (i), the volume \(V\) of an \(\ell\)-ball of radius \(\nu\delta(h)\) is proportional to \( (\nu\delta(h))^D\) as \(V_D^p(\nu\delta(h))=(2\nu\delta(h)\Gamma(1+1/p))^D/\Gamma(1+D/p)\).  Thus, the number of disjoint   \(\ell\)-balls of radius \(\nu\delta(h)\)  that fit in  \(X_{\delta(h)}\) is at most  \(\ceil{(\frac{\delta(h)}{\theta\nu\delta(h)})^D}  = \ceil{(\theta\nu)^{-D} }\). Therefore, the number of $\ell$-balls does not depend on $\delta(h)$ in this case, which means $d$ = 0. \vspace{-2pt}

Now that we have satisfied Assumptions 3, 4, and 5 with \(d=0\), \(\gamma=3^{-\alpha}\), and \(c_{1}=L3^\alpha D^{\alpha/p}\beta^\alpha\), we follow the proof of Corollary 1 and deduce the desired statement.
\end{proof}

\renewcommand\refname{References}

\bibliography{nips2015}

\begin{thebibliography}{10}

\bibitem{de2012exponential}
N.~De~Freitas, A.~J. Smola, and M.~Zoghi.
\newblock {Exponential regret bounds for Gaussian process bandits with
  deterministic observations}.
\newblock In {\em Proceedings of the 29th International Conference on Machine
  Learning (ICML)}, 2012.

\bibitem{wang2014bayesian}
Z.~Wang, B.~Shakibi, L.~Jin, and N.~de~Freitas.
\newblock {Bayesian Multi-Scale Optimistic Optimization}.
\newblock In {\em Proceedings of the 17th International Conference on
  Artificial Intelligence and Statistics (AISTAT)}, pages 1005--1014, 2014.

\bibitem{snoek2012practical}
J.~Snoek, H.~Larochelle, and R.~P. Adams.
\newblock {Practical Bayesian optimization of machine learning algorithms}.
\newblock In {\em Proceedings of Advances in Neural Information Processing
  Systems (NIPS)}, pages 2951--2959, 2012.

\bibitem{carter2001algorithms}
R.~G. Carter, J.~M. Gablonsky, A.~Patrick, C.~T. Kelley, and O.~J. Eslinger.
\newblock {Algorithms for noisy problems in gas transmission pipeline
  optimization}.
\newblock {\em Optimization and engineering}, 2(2):139--157, 2001.

\bibitem{zwolak2005globally}
J.~W. Zwolak, J.~J. Tyson, and L.~T. Watson.
\newblock {Globally optimised parameters for a model of mitotic control in frog
  egg extracts}.
\newblock {\em IEEE Proceedings-Systems Biology}, 152(2):81--92, 2005.

\bibitem{dixon1977global}
L.~C.~W. Dixon.
\newblock {\em {Global optima without convexity}}.
\newblock Numerical Optimisation Centre, Hatfield Polytechnic, 1977.

\bibitem{shubert1972sequential}
B.~O. Shubert.
\newblock {A sequential method seeking the global maximum of a function}.
\newblock {\em SIAM Journal on Numerical Analysis}, 9(3):379--388, 1972.

\bibitem{mayne1984outer}
D.~Q. Mayne and E.~Polak.
\newblock {Outer approximation algorithm for nondifferentiable optimization
  problems}.
\newblock {\em Journal of Optimization Theory and Applications}, 42(1):19--30,
  1984.

\bibitem{mladineo1986algorithm}
R.~H. Mladineo.
\newblock {An algorithm for finding the global maximum of a multimodal,
  multivariate function}.
\newblock {\em Mathematical Programming}, 34(2):188--200, 1986.

\bibitem{strongin1973convergence}
R.~G. Strongin.
\newblock Convergence of an algorithm for finding a global extremum.
\newblock {\em Engineering Cybernetics}, 11(4):549--555, 1973.

\bibitem{kvasov2003local}
D.~E. Kvasov, C.~Pizzuti, and Y.~D. Sergeyev.
\newblock {Local tuning and partition strategies for diagonal GO methods}.
\newblock {\em Numerische Mathematik}, 94(1):93--106, 2003.

\bibitem{bubeck2011lipschitz}
S.~Bubeck, G.~Stoltz, and J.~Y. Yu.
\newblock {Lipschitz bandits without the Lipschitz constant}.
\newblock In {\em Algorithmic Learning Theory}, pages 144--158. Springer, 2011.

\bibitem{gardner2014bayesian}
J.~Gardner, M.~Kusner, K.~Weinberger, and J.~Cunningham.
\newblock {Bayesian Optimization with Inequality Constraints}.
\newblock In {\em Proceedings of The 31st International Conference on Machine
  Learning (ICML)}, pages 937--945, 2014.

\bibitem{wang2013bayesian}
Z.~Wang, M.~Zoghi, F.~Hutter, D.~Matheson, and N.~De~Freitas.
\newblock Bayesian optimization in high dimensions via random embeddings.
\newblock In {\em Proceedings of the Twenty-Third international joint
  conference on Artificial Intelligence}, pages 1778--1784. AAAI Press, 2013.

\bibitem{srinivas2010gaussian}
N.~Srinivas, A.~Krause, M.~Seeger, and S.~M. Kakade.
\newblock {Gaussian Process Optimization in the Bandit Setting: No Regret and
  Experimental Design}.
\newblock In {\em Proceedings of the 27th International Conference on Machine
  Learning (ICML)}, pages 1015--1022, 2010.

\bibitem{murphy2012machine}
K.~P. Murphy.
\newblock {\em Machine learning: a probabilistic perspective}.
\newblock MIT press, page 521, 2012.

\bibitem{gpml06}
C.~E. Rasmussen and C.~Williams.
\newblock {\em {Gaussian Processes for Machine Learning}}.
\newblock MIT Press, 2006.

\bibitem{munos2011optimistic}
R.~Munos.
\newblock {Optimistic optimization of deterministic functions without the
  knowledge of its smoothness}.
\newblock In {\em Proceedings of Advances in neural information processing
  systems (NIPS)}, 2011.

\bibitem{jones1993lipschitzian}
D.~R. Jones, C.~D. Perttunen, and B.~E. Stuckman.
\newblock {Lipschitzian optimization without the Lipschitz constant}.
\newblock {\em Journal of Optimization Theory and Applications},
  79(1):157--181, 1993.

\bibitem{kandasamy2015high}
K.~Kandasamy, J.~Schneider, and B.~Poczos.
\newblock High dimensional {B}ayesian optimisation and bandits via additive
  models.
\newblock {\em arXiv preprint arXiv:1503.01673}, 2015.

\bibitem{simulationlib}
S.~Surjanovic and D.~Bingham.
\newblock Virtual library of simulation experiments: Test functions and
  datasets.
\newblock Retrieved November 30, 2014, from \url{http://www.sfu.ca/~ssurjano},
  2014.

\bibitem{mcdonald2007global}
D.~B. McDonald, W.~J. Grantham, W.~L. Tabor, and M.~J. Murphy.
\newblock Global and local optimization using radial basis function response
  surface models.
\newblock {\em Applied Mathematical Modelling}, 31(10):2095--2110, 2007.

\bibitem{walsh2010integrating}
T.~J. Walsh, S.~Goschin, and M.~L. Littman.
\newblock Integrating {Sample-Based} {Planning} and {Model-Based}
  {Reinforcement} {Learning}.
\newblock In {\em Proceedings of the 24th AAAI conference on Artificial
  Intelligence (AAAI)}, 2010.

\bibitem{strehl2009reinforcement}
A.~L. Strehl, L.~Li, and M.~L. Littman.
\newblock Reinforcement learning in finite {MDPs}: {PAC} analysis.
\newblock {\em The Journal of Machine Learning Research (JMLR)}, 10:2413--2444,
  2009.

\end{thebibliography}
\bibliographystyle{unsrt}

\end{document}